\documentclass[numbers]{article}

\usepackage[final]{neurips_2023}

\usepackage[utf8]{inputenc} 
\usepackage[T1]{fontenc}    
\usepackage{hyperref}       
\usepackage{url}            
\usepackage{booktabs}       
\usepackage{amsfonts}       
\usepackage{mathtools}
\usepackage{amsmath,amssymb,amsthm}
\usepackage{nicefrac}       
\usepackage{microtype}      
\usepackage{xcolor}         
\usepackage{bm}
\usepackage{mathabx,mathrsfs}
\usepackage[pdftex]{graphicx}
\usepackage{subcaption}
\newtheorem{theorem}{Theorem}
\newtheorem{lemma}{Lemma}

\newtheorem{example}{Example}

\newtheorem{corollary}{Corollary}
\theoremstyle{definition}
\newtheorem{remark}{Remark}
\newtheorem{definition}{Definition}
\newcommand{\A}{\mathcal{A}}
\newcommand{\p}{\mathbb{P}}
\newcommand{\E}{\mathbb{E}}
\newcommand{\R}{\mathbb{R}}
\newcommand{\RF}{\mathsf{RF}_B}
\newcommand{\RFi}{\mathsf{RF}_B^{\backslash i}}
\newcommand{\tree}{\mathsf{TREE}}
\newcommand{\rfall}{\mathsf{rf}}
\newcommand{\rfi}{\mathsf{rf}^{\backslash i}}

\title{Stability of Random Forests and Coverage of Random-Forest Prediction Intervals}

\author{%
  Yan Wang \\
  Department of Mathematics\\
  Wayne State University\\
  Detroit, MI 48202 \\
  \texttt{wangyan@wayne.edu} \\
   \And
   Huaiqing Wu, Dan Nettleton \\
   Department of Statistics \\
   Iowa State University \\
   Ames, IA 50011\\
   \texttt{\{isuhwu,dnett\}@iastate.edu} \\
}

\begin{document}

\maketitle

\begin{abstract}
  We establish stability of random forests under the mild condition that the squared response ($Y^2$) does not have a heavy tail. In particular, our analysis holds for the practical version of random forests that is implemented in popular packages like \texttt{randomForest} in \texttt{R}. Empirical results show that stability may persist even beyond our assumption and hold for heavy-tailed $Y^2$. Using the stability property, we prove a non-asymptotic lower bound for the coverage probability of prediction intervals constructed from the out-of-bag error of random forests. With another mild condition that is typically satisfied when $Y$ is continuous, we also establish a complementary upper bound, which can be similarly established for the jackknife prediction interval constructed from an arbitrary stable algorithm. We also discuss the asymptotic coverage probability under assumptions weaker than those considered in previous literature. Our work implies that random forests, with its stability property, is an effective machine learning method that can provide not only satisfactory point prediction but also justified interval prediction at almost no extra computational cost.  
\end{abstract}

\section{Introduction}
Random forests (RFs) is a successful machine learning method that serves as a standard approach to tabular data analysis and has good predictive performance \cite{breiman2001random,biau2016random}. However, there is a big gap between the empirical effectiveness of RFs and the limited understanding of its properties. Most known theoretical results are established for variants of RFs not necessarily used in practice \cite{biau2012analysis,scornet2016random,lin2006random,denil2014narrowing,tang2018random}. For the RF version implemented in packages like \texttt{randomForest} in \texttt{R} \cite{liaw2002classification}, little is known without strong assumptions \cite{biau2008consistency,scornet2015consistency,zhang2019random}; RFs is notoriously difficult to analyze as a greedy algorithm. Here we show an important property for the RF used in practice (as well as for other variants) under realistic conditions.

\subsection{Stability of random forests}
The first main contribution of this work establishes the stability condition for the RF.
\begin{theorem}[Stability of random forests,  informal]\label{thm:RF_stability_informal}
    For independent and identically distributed (iid) training data points $(X_i,Y_i),i\in\{1,\ldots,n\}\equiv[n]$ and a test point $(X,Y)$, if the squared response $Y^2$ does not have a heavy tail, then the RF predictor $\RF$ and any out-of-bag (OOB) predictor $\RFi$ predict similar values, i.e.,
    \begin{align}\label{eq:informal}
        \p\left(\left| \RF(X) - \RFi(X) \right| > \varepsilon_{n,B}\right) \leq \nu_{n,B},
    \end{align}
    where $\RF$ results from the aggregation of all $B$ base tree predictors, while $\RFi$ only those with the point $(X_i,Y_i)$ excluded in training; $\varepsilon_{n,B}$ and $\nu_{n,B}$ are small numbers depending on $n$ and $B$.    
\end{theorem}
This result is referred to as the stability of the RF because it indicates that no single training point is extremely important in determining $\RF$ in a probabilistic sense. Theorem \ref{thm:RF_stability_informal} relies on a recent important work that establishes the absolute stability (see below for a precise definition) of general bagged algorithms with bounded outputs \cite{soloff2023bagging}. We take advantage of the fact that the range of the RF output is conditionally dependent upon the maximal and minimal values of $Y$ in the training set, and then we show in theory that the stability property of the RF is possible even if $Y$ is marginally unbounded. To our knowledge, this is the first stability result established for the RF. 

The technique used in our analysis requires that $Y^2$ not have a heavy tail (to make $\varepsilon_{n,B}$ and $\nu_{n,B}$ small). Though arguably already mild, we conjecture that this condition might be further relaxed. As shown below, numerical evidence suggests that the light-tail assumption may not be necessary for RF stability, which could hold even when $Y$ follows a heavy-tail distribution like the Cauchy distribution. 

\subsection{Random-forest prediction intervals}
Stability is a crucial property of a learning algorithm. For example, stability has a deep connection with the generalization error of an algorithm \cite{bousquet2002stability,kontorovich2014concentration,maurer2021concentration}. Moreover, stability also turns out to be important in distribution-free predictive inference. In particular, an algorithm being stable justifies the jackknife prediction interval (PI), which otherwise has no coverage guarantee \cite{barber2021predictive}. 

In this work, we show that stability makes it possible to construct a PI with guaranteed coverage from the OOB error of the RF. The OOB error is defined as $R_i = |Y_i-\RFi(X_i)|$, $i\in[n]$. A main reason why such a PI is appealing is that $R_i$ can be obtained almost without extra effort. For example, a one-shot training using the \texttt{R} package \texttt{randomForest} gives us an RF predictor $\RF$ and all $n$ OOB predictions $\RFi(X_i)$. So, from the computational point of view, a convenient way to construct a PI for a test point $(X,Y)$ is of the form $``\RF(X)\pm \text{proper quantile of } \{R_i\}"$ \cite{johansson2014regression,zhang2019random}.

The second main contribution of this work constructs such PIs and theoretically proves, under mild conditions, the non-asymptotic lower and upper bounds for the coverage probability. \begin{theorem}[Coverage lower bound, informal]
Under the same assumptions as in Theorem \ref{thm:RF_stability_informal}, and for $\alpha\in(0,1)$ \footnote{When $\alpha\in(0,1/(n+1))$, we follow the convention that the $(n+1)$-th smallest $R_i$ is $\infty$.}, we have the following lower bound of coverage probability:
    \begin{align*}
        \p\left(
        |Y-\RF(X)|\leq \text{the $\lceil (n+1)(1-\alpha) \rceil$-th smallest $R_i$} + \varepsilon_{n,B}\right) \geq 1-\alpha-O(\sqrt{\nu_{n,B}}),
    \end{align*}
where $\lceil\cdot\rceil$ is the ceiling function. Big $O$ and other related notations are used in the usual way. 
\end{theorem}
\begin{theorem}[Coverage upper bound, informal]
If we further assume that $Y$ is continuous, resulting in distinct prediction errors, then we also have the following upper bound: 
    \begin{align*}
        \p\left(
        |Y-\RF(X)|\leq \text{the $\lceil (n+1)(1-\alpha) \rceil$-th smallest $R_i$} - \varepsilon_{n,B}\right) \leq 1-\alpha+\frac{1}{n+1}+O(\sqrt{\nu_{n,B}}).
    \end{align*}
\end{theorem}
As we detail below, the PIs we provide coverage guarantees for are neither the jackknife-with-stability interval discussed in \cite{barber2021predictive}, nor the jackknife+-after-bootstrap interval established in \cite{kim2020predictive}. In our context, constructing the former needs $n$ leave-one-out (LOO) predictors (rather than $n$ OOB predictors), i.e., $n$ additional RFs with each built on a training set of size $n-1$. Constructing the latter needs the explicit information of each $\RFi(\cdot)$ rather than the OOB prediction $\RFi(X_i)$ for each $X_i$ only. Both these methods require additional, sometimes extensive, computation given current popular packages. In contrast, our results are operationally more convenient. After one-shot training, we obtain not only a point predictor $\RF(\cdot)$, but also a valid interval predictor at almost no extra cost. Under reasonable conditions, our results indicate that by slightly inflating (or deflating) the PI constructed from the $\lceil (n+1)(1-\alpha) \rceil$-th smallest $R_i$, the coverage probability is guaranteed not to decrease (or increase) too much from the desired level of $1-\alpha$. In fact, many numerical results, such as those in \cite{johansson2014regression,zhang2019random}, suggest that
\[
\p\left(
        |Y-\RF(X)|\leq \text{the $\lceil (n+1)(1-\alpha) \rceil$-th smallest $R_i$}\right) \approx 1-\alpha.
\]
Motivated by this fact, we further establish an asymptotic result of coverage for such PIs.
\begin{theorem}(Asymptotic coverage, informal)
    In addition to the conditions in the above theorems, also suppose the prediction error $|Y-\RF(X)|$ is continuous, and its cumulative distribution function (CDF) does not change too drastically for all sufficiently large $n$. Then 
    \[
        \p\left(
        |Y-\RF(X)|\leq \text{the $\lceil (n+1)(1-\alpha) \rceil$-th smallest $R_i$}\right) \to 1-\alpha \text{ as $n\to\infty$.}
    \]
\end{theorem}
In \cite{zhang2019random}, this asymptotic coverage was proved based on stronger assumptions. In particular, the true model is assumed to be additive such that $``Y=f_0(X)+\text{noise}"$ with the zero-mean noise independent of $X$, and $\RF(X)$ is assumed to converge to $f_0(X)$ in probability. We do not require $\RF$ to converge to anything in any sense when $n\to\infty$. Technically, we need the family of prediction error CDFs be uniformly equicontinuous. 

Based on our results, the RF seems to be the only one, among existing popular machine learning algorithms, that can provide both point and interval predictors with justification in such a convenient way. This makes the RF appealing, especially for tasks where the computational cost is a concern.    

It is also worth noting that the upper-bound result is of interest in its own right. It can be generalized to jackknife PIs that are constructed from any stable algorithm; the result serves as a complement to the lower bounds established previously \cite{barber2021predictive,kim2020predictive}. 

Summarizing, we 
\begin{itemize}
    \item theoretically prove that the (greedy) RF algorithm is stable when  $Y^2$ does not have a heavy tail;
    \item numerically show that RF stability may hold beyond the above light-tail assumption; 
    \item construct PIs based on the OOB error with finite-sample coverage guarantees: the lower bound of coverage does not need any additional assumption beyond stability; the upper bound needs an additional assumption, which is usually satisfied when $Y$ is continuous;
    \item provide the upper bound of coverage for jackknife PIs constructed from general stable algorithms, assuming distinct LOO errors; and
    \item prove asymptotically exact coverage for RF-based PIs under weaker assumptions than those previously considered in published work.

\end{itemize}

\section{Concepts of algorithmic stability}

Stability stands at the core of this work. There are different types of stability, each of which is used to assess quantitatively how stable (in some certain sense) an algorithm is with respect to small variations in training data \cite{bousquet2002stability,soloff2023bagging,bertsimas2022stable}. In a recent work \cite{bertsimas2022stable}, robust optimization is used to enhance the stability of algorithms in classification tasks. In \cite{soloff2023bagging}, bagging is proved to be an efficient mechanism to stabilize algorithms in regression tasks. We focus on regression here. As will be made clear, the technique used in this work relies on the fact that the RF predictor in regression results from averaging tree predictors. However, the majority vote of tree predictors is used in classification, and new ideas are needed to analyze the RF stability in this setting. For our purposes, we introduce three levels of stability from strongest to weakest. The strongest version of stability, introduced in \cite{soloff2023bagging}, does \emph{not} depend on the data distribution, and may be referred to as ``absolute stability.''
\begin{definition}[Absolute stability of algorithms]
For any dataset consisting of $n\geq 2$ training points $D=\{(X_1,Y_1),\ldots,(X_n,Y_n)\}$ and any test point $(X,Y)$, an algorithm $\A$ is defined to be $(\varepsilon,\delta)$-absolutely-stable if 
\begin{align*}
    \frac{1}{n}\sum_{i=1}^n \p_{\xi}\left(\left|\hat f (X) - \hat f^{- i}(X)\right|>\varepsilon\right) \leq \delta
\end{align*}
for some $\varepsilon,\delta\geq 0$, where $\xi$ denotes the possible innate randomness in the algorithm (such as the node splitting procedure in the RF) and can be seen as a random variable uniformly distributed in $[0,1]$, $\hat f=\A(D;\xi)$ is the predictor trained on $D$, and $\hat f^{- i}=\A(D^{- i};\xi)$ is the $i$th LOO predictor trained on $D^{- i}$, i.e., $D$ without the $i$th point $(X_i,Y_i)$. We might refer to the RF as both an algorithm (the learning procedure) and a predictor (the learned function) for simplicity.  
\end{definition}
Many bagged algorithms, in particular those with bounded predicted values, can achieve absolute stability with both $\varepsilon$ and $\delta$ converging to 0, as long as $n$ and the number of bags $B$ go to infinity. However, the predicted value of the RF is in general unbounded (for regression tasks considered in this work), and we are more interested in another type of stability, investigated in \cite{bousquet2002stability}, and called out-of-sample stability \cite[][]{barber2021predictive}. For simplicity, we name it ``stability.'' This notion of stability turns out to be important in validating a jackknife prediction interval.
\begin{definition}[Stability of algorithms]
For iid training and test data, algorithm $\A$ is $(\varepsilon,\delta)$-stable if
\begin{align*}
     \p_{D,X,\xi}\left(\left|\hat f (X) - \hat f^{- i}(X)\right|>\varepsilon\right) \leq \delta
\end{align*}   
for some $\varepsilon,\delta\geq 0$, where $D,X,\hat f, \hat f^{-i}$ are as defined above.
\end{definition}
We will establish this type of stability for the derandomized RF defined below, where the data-generating distribution is involved. To this end, we will use the methods in \cite{soloff2023bagging}, which aim to provide absolute stability for bagged algorithms. Technically, we use such methods to first establish the ``conditional stability'' of an algorithm with respect to given data.
\begin{definition}[Conditional stability of algorithms]
Conditional on $D$ and $X$, an algorithm $\A$ is defined to be $(\varepsilon,\delta)$-conditionally-stable if 
\begin{align*}
    \frac{1}{n}\sum_{i=1}^n \p_{\left. 
    \xi\middle| D,X \right.}
    \left(
    \left.
    \left|\hat f (X) - \hat f^{- i}(X)\right|>\varepsilon
    \middle | D,X
    \right.
    \right) \leq \delta
\end{align*}
for some $\varepsilon,\delta\geq 0$, where $D,X,\hat f, \hat f^{- i}$ are as defined above.   
\end{definition}
Once conditional stability is established for the derandomized RF algorithm, its stability can be consequently established by invoking 
\[
\p_{D,X,\xi}(\cdot)=\E_{D,X}\left[\p_{\xi|D,X}(\cdot|D,X)\right]. 
\]
Stability of the derandomized RF provides the most essential ingredient for that of the practical RF, although the latter involves another type of stability, known as ensemble stability \cite{kim2020predictive}. Ensemble stability justifies replacing the LOO predictor with the OOB predictor in (\ref{eq:informal}). We may abuse the term ``stability'' in the following when the OOB, rather than the LOO, predictor is used.

\section{Stability of random forests}\label{sec:RF_stability}  
\subsection{Basics of random forests}

This work mainly considers using the RF to perform regression tasks, where the response $Y\in\R$ can be unbounded. By construction, the RF predictor with $B$ bags, denoted by $\RF$, is a bagged algorithm with the base algorithm being a tree, and  $\RF = \frac{1}{B}\sum_{b=1}^B \tree_b,$ where $\tree_b$ is the $b$th tree predictor, trained on the $b$th bag $r_b$, a bootstrapped sample of the training set $D$. The randomness in the tree predictor $\tree$ originates from two independent sources: innate randomness $\xi$ in the node splitting process and resampling randomness from the bag $r$. For the $i$th point, one can define the OOB RF predictor as $\RFi = \frac{1}{B_i} \sum_{b=1}^B \tree_b \times \mathbb{I}\{i\notin r_b\},$ where $\mathbb{I}\{\cdot\}$ denotes the indicator function, and $B_i = \sum_{b=1}^B \mathbb{I}\{i\notin r_b\}.$ Define $p\equiv \p(i\in r)$ as the probability that the $i$th point is included in bag $r$. Then it is clear that $B_i\sim\mathrm{Binomial}(B,1-p)$ for all $i$. We also denote $\rfall$ and $\rfi$ as the \emph{derandomized} versions of $\RF$ and $\RFi$, respectively. Precisely, $\rfall = \E_{\xi,r} [\tree] 
\text{ and } 
\rfi = \E_{\xi,r} [\tree | i\notin r].$ It is worth noting that, by definition, $\RFi\neq \RF^{-i}$ for finite $B$, while $\rfi=\rfall^{-i}$ as the derandomized RF results from the aggregation of an infinite number of trees. Since RF predictors are averages over tree predictors, the predicted values they output, given training set $D$, are bounded in $[Y_{(1)},Y_{(n)}]$, where $Y_{(1)}$ and $Y_{(n)}$ are the minimum and maximum of $\{Y_1,\ldots,Y_n\}$, respectively. We also let $Z_i=|Y_i|$ for all $i$, and denote the maximum as $Z_{(n)}$. As a result, we have that
\begin{align}\label{eq:RF_range}
    |\rfi - \rfall| \leq Y_{(n)} - Y_{(1)}\leq 2Z_{(n)}.
\end{align}
\begin{remark}
This is also true for $\RFi$ and $\RF$ for any finite $B$. In fact, this is a distinctive feature of the RF, \emph{irrespective} of the node splitting rule. Other regression methods do not necessarily have such a data-dependence bound. This observation helps to establish the conditional stability of the RF.
\end{remark}
\begin{remark}
     Practically, when $n$ is large, one might think that the bound (\ref{eq:RF_range}) is crude. On one hand, if we look for a bound valid for any finite $n\geq 2$, then there is not much room for improvement for small $n$. On the other hand, we do expect that the \emph{typical} stability of the RF can go beyond the finite-sample guarantee provided by (\ref{eq:RF_range}) when $n$ is big, which is consistent with the numerical results shown below. A more informative bound for large $n$ is worth future investigation.
\end{remark}

There are several quantities that are useful in establishing the RF stability; they can be calculated explicitly and are listed below. First, it is well known that
\begin{align}\label{eq:p}
p \equiv \p(i\in r) = 1-(1-1/n)^n = 1- 1/e+O(1/n).
\end{align}
Actually, $p$ is monotonically decreasing for $n\geq 1$. Second, 
\begin{align}\label{eq:q}
q &{} \equiv -\mathrm{Cov}(\mathbb{I}\{i\in r\},\mathbb{I}\{j\in r\})  = (1-1/n)^{2n} - (1-2/n)^n = O(1/n),
\end{align}
as can be directly checked. Third, the moment generating function of $B_i$ is
\begin{align}\label{eq:BiMGF}
\E\left[e^{sB_i}\right] = (p+(1-p)e^s)^B.
\end{align}

In the following, we first perform the stability analysis for the derandomized RF (consisting of an infinite number of trees) and then extend the results to the practical finite-$B$ case. 

\subsection{Derandomized random forests}
The following theorem formalizes the conditional stability property for the derandomized RF, the proof of which is a direct result of Theorem 8 in \cite{soloff2023bagging}, and is omitted here.
\begin{theorem}[Conditional stability of derandomized random forests]
    Conditional on training set $D$ and test point $(X,Y)$, for the derandomized random forest predictor $\rfall$ we have that
\begin{align}\label{eq:rf_conditional_stability}
    \frac{1}{n}\sum_{i=1}^n
    \mathbb{I}\left\{
    \left.
    \left| \rfall(X) - \rfi(X) \right| > \varepsilon 
    \middle | D,X
    \right.
    \right\} \leq 
    \delta(D,X) \equiv
    \frac{Z_{(n)}^2}{\varepsilon^2n}
    \left(\frac{p}{1-p}+\frac{q}{(1-p)^2}\right).
\end{align}
\end{theorem}
If $\delta(D,X)\geq 1$, the statement is trivial, and we will focus on the case that $\delta(D,X)\in(0,1)$. We can now establish the stability property for the derandomized RF.

\begin{theorem}[Stability of derandomized random forests]\label{thm:rf_stability}
    For iid training and test data and $\varepsilon>0$, the derandomized random forest predictor $\rfall$ is stable with
    \begin{align}\label{eq:rf_stability}
       \p_{D,X}\left(
       \left| \rfall(X) - \rfi(X) \right| > \varepsilon
       \right) 
       \leq  
       \frac{\E[Z_{(n)}^2]}{\varepsilon^2n}
       \left(\frac{p}{1-p}+\frac{q}{(1-p)^2}\right) \equiv \nu.
    \end{align}
\end{theorem}
This result follows directly from the conditional stability (\ref{eq:rf_conditional_stability}) by averaging over $D$ and $X$. There is some freedom in choosing the dependence of $\varepsilon$ on $n$. On one hand, in order to make sense of the word ``stability,'' we do expect $\varepsilon$ and $\nu$ to be small for large $n$. From (\ref{eq:p}) and (\ref{eq:q}), it is clear that the asymptotic behavior of $\nu$ is dominated by that of $\E[Z_{(n)}^2]/(\varepsilon^2 n)$, which can be tuned by manipulating $\varepsilon$. For example, a matching convergence rate to 0 between $\varepsilon$ and $\nu$ might be desirable, and one can then set $\varepsilon=O((\E[Z_{(n)}^2]/n)^{1/3})$ if the scaling of $\E[Z_{(n)}^2]=o(n)$ is known or can be inferred. On the other hand, we can fix $\varepsilon$ to further investigate the relation between stability and the convergence-in-probability property of the RF. By (\ref{eq:rf_stability}), under the condition that $\E[Z_{(n)}^2]/n\to 0$ as $n\to\infty$, one immediately comes to the conclusion that $\rfi(X)-\rfall(X)$ converges to 0 in probability. Actually, a stronger conclusion can be drawn under the same condition.

\begin{corollary}\label{cor:rf_consistency}
   For iid training and test data, we have 
   \begin{align}
       \E_{D,X}[|\rfall(X)-\rfi(X)|] < 2\sqrt{\frac{\E[Z_{(n)}^2]}{n}\left(\frac{p}{1-p}+\frac{q}{(1-p)^2}\right)}.
   \end{align}
   Further assume that $\E[Y^2]<\infty$. Then we have
   \begin{align}\label{eq:rf-rfi}
       \E_{D,X}[|\rfall(X)-\rfi(X)|]\to 0 \text{ and } \rfall(X) - \rfi(X) \overset{\p}{\to} 0 \text{ as } n\to\infty.
   \end{align}
\end{corollary}
\begin{remark}
 The additional assumption that $\E[Y^2]<\infty$ is mild. Many commonly encountered random variables  have a light tail and thus a finite second moment, irrespective of the detailed information of the distribution in question. Note that the bound (\ref{eq:RF_range}) itself can be crude, and our result is expected to be valid even beyond this mild condition. 
\end{remark}
\begin{remark}
This result indicates that the \emph{difference} is diminishing between $\rfall$ and $\rfi$, built on $n$ and $n-1$ training data points, respectively. However, there is no indication that the derandomized $\rfall(X)$ itself will converge to anything. This idea inspires the proposal of Theorem \ref{thm:PI_asymptotic}.  
\end{remark}
The proof of this result, as well as others below, will be deferred to the Appendix. So far, we have investigated the derandomized version of the RF, which is a limiting case and can be seen as consisting of an infinite number of trees, averaging out all kinds of possible randomness in the predictor construction process. In order to make the results more relevant to applied machine learning, the finite-$B$ analysis for the RF is conducted below.   

\subsection{Finite-$B$ random forests}
We now consider the difference between $\RF$ and $\RFi$. We denote $\bm{\xi}=(\xi_1,\ldots,\xi_B)$ and $\bm{r}=(r_1,\ldots,r_B)$ as the corresponding sources of randomness in $B$ trees. We also consider conditional stability first and then move to the stability of $\RF$. 
\begin{theorem}[Conditional stability of finite-$B$ random forests]\label{thm:RF_conditional_stability} Conditional on training set $D$ and test point $(X,Y)$, for a random forest predictor $\RF$ that consists of $B$ trees, we have for $\varepsilon>0$ that
\begin{align*}
 \frac{1}{n}\sum_{i=1}^n\p_{\bm{\xi,r}|D,X}
     \left(
     \left.
    \left|\RF(X)-\RFi(X)\right|>\varepsilon + 2\sqrt{\frac{2Z_{(n)}^2}{B}\ln\left(\frac{1}{\delta}\right)}
    \middle | D,X
    \right.
    \right)  
    \leq 3\delta+g(p,\delta,B),
\end{align*} 
where $\delta$ is short for $\delta(D,X)$ as defined in (\ref{eq:rf_conditional_stability}) and $g(p,\delta,B)=2(p+(1-p)\delta^{\frac{1}{B}})^B$. 
\end{theorem}

Next, we consider the case of iid data and investigate the RF stability by averaging out the randomness in data. Note that $Z_{(n)}$ and $\delta$ are random and depend on the data distribution, while we are interested in a probability bound for $|\RF(X)-\RFi(X)|$ greater than a deterministic quantity, which is only a function of $B$ and $n$. In this finite-$B$ case, the stability of $\RF$ cannot be directly obtained from its conditional stability as in the derandomized situation.

\begin{theorem}[Stability of finite-$B$ random forests]\label{thm:RF_stability}
    Assume training points in set $D$ and the test point $(X,Y)$ are iid, drawn from a fixed distribution. For the random forest predictor $\RF$ consisting of $B$ trees and trained on $D$, we have 
    \begin{align}
    \label{eq:RF_stability}
    \p_{D,X,\bm{\xi,r}}
      \left(
      \left|\RF(X)-\RFi(X)\right|>\varepsilon_{n,B}
      \right) 
    \leq \nu_{n,B},
    \end{align}
where $\varepsilon_{n,B}=\sum_{i=1}^3\varepsilon_i$, and $\nu_{n,B}=\sum_{i=1}^3\nu_i$. The pair of $(\varepsilon_2,\nu_2/\lambda)$ satisfies the derandomized RF stability condition (\ref{eq:rf_stability}) with $\lambda>1$. Moreover, $\varepsilon_1=\varepsilon_3=\sqrt{2\lambda\E[Z_{(n)}^2]\ln(\frac{1}{\nu_2})/B}$, $\nu_1=2\nu_2+2\p( Z_{(n)}^2 > \lambda\E[Z_{(n)}^2])$, and $\nu_3=g(p,\nu_2,B) + 2\p( Z_{(n)}^2 > \lambda\E[Z_{(n)}^2])$.
\end{theorem}
On a high level, the establishment of this theorem relies on two observations: (i) the stability of the derandomized RF, so that the difference $|\rfall(X)-\rfi(X)|$ is controlled, and (ii) the concentration of measure, so that the differences $|\RF(X)-\rfall(X)|$ and $|\RFi(X)-\rfi(X)|$ are controlled. In order to make full sense of the word ``stability,'' it is desirable that $\varepsilon_{n,B}$ and $\nu_{n,B}$ can converge to 0. It is known that 
$\E[Y^2]<\infty$ suffices to ensure $\E[Z_{(n)}^2]=o(n)$ \cite{downey1990distribution,correa2021asymptotic}, and hence the stability of the derandomized RF. Now in the finite-$B$ case, we need an additional distributional assumption to control the tail probability $\p(Z_{(n)}^2>\lambda \E[Z_{(n)}^2])$. It turns out that for typical light-tailed $Y^2$, such a tail probability will converge to 0 as $n\to\infty$. Technically, we can assume $Y^2$ to be sub-gamma \cite{boucheron2013concentration}. Note that bounded and sub-Gaussian random variables are sub-gamma. Hence the sub-gamma assumption is not strong and can be satisfied by distributions underlying many real datasets.  
\begin{definition}[Sub-gamma random variables \cite{boucheron2013concentration}]
    A random variable $W$ is said to be sub-gamma (on the right tail) with parameters $(\sigma^2,c)$ where $c\geq 0$, if $\ln\E[ e^{s(W-\E[W])}]\leq \frac{s^2\sigma^2}{2(1-cs)} \text{ for all $s\in(0,1/c)$.}$
\end{definition}
\begin{lemma}\label{lem:sub-gamma_tail}
Suppose $Y^2$ is sub-gamma with parameters $(\sigma^2, c)$ with $c>0$, and $\E[Z_{(n)}^2]\sim a\ln n$ with $a\leq c$. For $\lambda > c/a$, we have $\lim_{n\to\infty} \p(Z_{(n)}^2>\lambda \E[Z_{(n)}^2]) = 0$.
\end{lemma}
\begin{remark}
    We have set $c>0$ above. If $c=0$, then $Y^2$ is in fact sub-Gaussian, and the tail probability can be controlled similarly. If $Y^2$ is upper bounded by some constant $M^2$, the stability analysis is even simpler, and there is no need to consider the tail probability at all, as we can use $M^2$ in place of $Z_{(n)}^2$ in the conditional stability of the RF and then take expectation with respect to data.   
\end{remark}
\begin{example}
Consider $Y^2 \sim \mathsf{Exp}(1)$, the exponential distribution with scale parameter 1. It is known that $Y^2$ is sub-gamma with $(\sigma^2,c)=(1,1)$ {\normalfont\cite{boucheron2013concentration}}, and $\E[Z_{(n)}^2]=\sum_{i=1}^n\frac{1}{i}\equiv H_n$ with $H_n\in(\gamma + \ln n, \gamma + \ln(n+1))$, where $\gamma\approx 0.577$ is Euler's constant. Hence $H_n = \ln n + o(\ln n)$, and a straightforward calculation reveals that $\lim_{n\to\infty} \p(Z_{(n)}^2>\lambda\E[Z_{(n)}^2]) = 0$ as long as $\lambda>1$.
\end{example}
From such results, one can see that the vanishing tail probability is not a stringent condition. By taking this additional assumption, it is indeed possible that both $\varepsilon_{n,B}$ and $\nu_{n,B}$ converge to 0.
\begin{corollary}\label{cor:epsilon_nu_to_0}
    For the same setting as in Theorem \ref{thm:RF_stability}, suppose $Y^2$ is sub-gamma with parameters $(\sigma^2,c)$ with $c>0$ and $\E[Z_{(n)}^2]\sim a\ln n$ with $a\leq c$. Let $\lambda > c/a$ be a fixed number, and let $B$ depend on $n$. Then for $\varepsilon_2$ that satisfies both $\varepsilon_2 = \omega(\sqrt{\ln n/n}) $ and $\varepsilon_2=o(1)$, and $B=\Omega(\ln^2n)$, we have     $\lim_{n\to\infty}\varepsilon_{n,B} =  \lim_{n\to\infty} \nu_{n,B} = 0.$  
\end{corollary}
It is worth noting that there are multiple ways to let $\varepsilon_{n,B}$ and $\nu_{n,B}$ approach 0, as the dependence of $\varepsilon_2$, $B$, and even $\lambda$ on $n$ can all be manipulated. The point is that, theoretically, even the greedy RF can be stable with vanishing parameters. In practice, however, the stability of $\RF$ seems to hold in broader situations where both the moment and tail assumptions on $Y^2$ can be relaxed.

\begin{figure}
  \centering
  \begin{subfigure}{0.5\textwidth}
  \includegraphics[width=\linewidth]{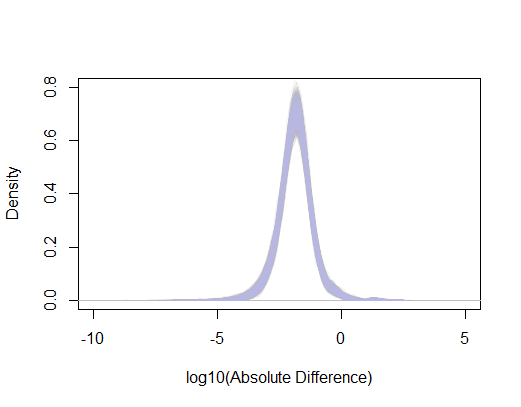}
  \end{subfigure}%
  \begin{subfigure}{0.5\textwidth}
  \includegraphics[width=\linewidth]{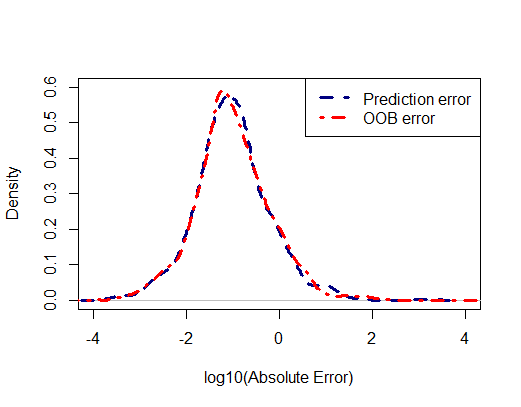}    
  \end{subfigure}
  \caption{Left: Density plots of the $\log_{10}$ absolute difference $|\RF(X)-\RFi(X)|$ for 3000 OOB predictors $\RFi$ on 1000 test points. We let $B=1000$. The RF stability (\ref{eq:RF_stability}) seems to persist, even though $Y$ follows the (heavy-tailed) standard Cauchy distribution. Numerically, we set $\hat \nu_{n,B}=0.05$ and calculated the maximum of the 0.95 quantile of the 3000 empirical distributions to have $\hat \varepsilon_{n,B}=0.237$. Right: Density plots of 1000 $\log_{10}$ absolute prediction errors $|Y-\RF(X)|$ and of 3000 $\log_{10}$ absolute OOB errors $|Y_i-\RFi(X_i)|$. The similarity between the plots supports the idea that the OOB errors can be used to construct PIs.}\label{fig:compare}
\end{figure}

\subsection{Stability in practice and limitations of theory}
We created a virtual dataset consisting of $n=4000$ points. We let $Y$ be a standard Cauchy random variable, which is even without a well-defined mean. The feature vector $X\in\R^3$ is determined as $X=[0.5Y+\sin(Y), Y^2-0.2Y^3, \mathbb{I}\{Y>0\}+\zeta]^T$ where $\zeta$ is a standard normal random variable. We used $3000$ of the points for training and $1000$ of them as test points. Using the \texttt{randomForest} package with default setting (except letting $B=1000$), we had an output RF predictor $\RF$. We also aggregated corresponding tree predictors to have $3000$ OOB predictors $\RFi$. For each $i\in[3000]$, we calculated the absolute difference $|\RF(X)-\RFi(X)|$ on $1000$ test points to come up with a density plot for such a difference, shown in Fig. \ref{fig:compare}. We also calculated $1000$ absolute prediction errors $|Y-\RF(X)|$ that are incurred by $\RF$ on test points, and $3000$ OOB errors $|Y_i-\RFi(X_i)|$, each incurred by an OOB predictor $\RFi$ on its OOB point $(X_i,Y_i)$. The computation can be done within a few minutes on a laptop. The density plots of these two kinds of errors are also shown in Fig. \ref{fig:compare}. This example shows that  the RF stability can be present beyond the realm guaranteed by the light-tail assumption. As mentioned above, this is because the bound (\ref{eq:RF_range}) can be conservative when $n$ is large. We hope our results can inspire future study towards a more informative bound. Also, the similarity between the prediction error and the OOB error in this heavy-tail case indicates that the RF-based PIs analyzed below can find more applications in practice than justified by the current theory.

\section{Random-forest prediction intervals}
\subsection{Comparison with related methods}
With the stability property of the RF, it is possible to construct PIs with finite-sample guaranteed coverage. Recent years have witnessed the development of distribution-free predictive inference \cite{angelopoulos2023conformal} with the full \cite{vovk2005algorithmic,shafer2008tutorial}, split \cite{papadopoulos2002inductive,vovk2012conditional,lei2018distribution}, and jackknife+ \cite{barber2021predictive,vovk2015cross} conformal prediction methods being three milestones. The full conformal method is computationally prohibitive when used in practice. The split method greatly reduces the computational cost but fails to thoroughly extract the available information of training data. The jackknife+ (\textbf{J+}) method maximizes the usage of data at a computational cost in between those of full and split methods. In \cite{kim2020predictive}, jackknife+-after-bootstrap (\textbf{J+aB}) was proposed for bagged algorithms to achieve the same goal as in J+, while the training cost can be further reduced. However, the number of bags $B$ is required to be a Binomial random variable, which might seem unnatural. It turns out that by further imposing the assumption of \emph{ensemble} stability (which is essentially the concentration of resampling measure), J+aB can still have guaranteed coverage with a fixed $B$. Ensemble stability is defined for bagged algorithms. It measures how typical a bootstrap sample is, and is different from the algorithmic stability that quantifies the influence of removing one training point. If algorithmic stability is also imposed, then not only J+aB, but also jackknife can provide guaranteed coverage, which is otherwise impossible \cite{steinberger2023conditional,barber2021predictive}. 

Conceptually, the J+ approach and its variants under stability conditions are particularly relevant to this work. As the stability we establish for the RF contains both ensemble and algorithmic components, we will generally refer to the J+aB method with both ensemble and algorithmic stability as \textbf{J+aBS} and the jackknife method with algorithmic stability as \textbf{JS}. Our method might be best described as ``jackknife-after-bootstrap-with-stability (\textbf{JaBS})'' tailored for the RF, which is different from both JS and J+aBS. Our method requires the least effort of computing as only one output predictor is needed, while all others require at least $n$ output predictors.

There also exist RF-based PIs \cite{johansson2014regression,zhang2019random} that are essentially of the jackknife-after-bootstrap (\textbf{JaB}) type and almost identical to ours practically when $\varepsilon$ is small and $n$ equals the size of a typical dataset. However, without stability, there is, in general, no guarantee for the coverage of such PIs, although the asymptotic coverage $1-\alpha$ can be established based on strong assumptions \cite{zhang2019random}. We take advantage of the stability of the RF algorithm to establish the lower bound of coverage in Theorem \ref{thm:PI_lower_bound} below. An upper bound is established in Theorem \ref{thm:PI_upper_bound} with an additional mild assumption. We also propose a weaker assumption for asymptotic coverage in Theorem \ref{thm:PI_asymptotic}.

\begin{table}
  \caption{Methods to construct prediction intervals using random forests: computational cost}
  \label{table:compare}
  \centering
  \begin{tabular}{lll}
    \toprule
    \cmidrule(r){1-3}
    Method     & Output predictors     & Prediction interval for future $Y$\\
    \midrule
    J+ \cite{barber2021predictive} & $\RF^{-i}, i\in[n]$  & $[q_{n,\alpha}^{-}\{\RF^{-i}(X)-R_i^{\mathrm{LOO}}\}, q_{n,\alpha}^{+}\{\RF^{-i}(X)+R_i^{\mathrm{LOO}}\}]$   \\
    
    J+aB \cite{kim2020predictive}     & $\RFi, i\in[n]$ & $[q_{n,\alpha}^{-}\{\RFi(X)-R_i\}, q_{n,\alpha}^{+}\{\RFi(X)+R_i\}]$    \\
    
    JS \cite{barber2021predictive}    & $\RF$ and $\RF^{-i}, i\in[n]$       & $\RF(X)\pm q_{n,\alpha}\{R_i^{\mathrm{LOO}}+\varepsilon\}$ \\
    
    J+aBS \cite{kim2020predictive}     & $\RFi, i\in[n]$       & $[q_{n,\alpha}^{-}\{\RFi(X)-R_i\}-\varepsilon, q_{n,\alpha}^{+}\{\RFi(X)+R_i\}+\varepsilon]$ \\
    
    JaB  & $\RF$ & $\RF(X)\pm q_{n,\alpha}\{R_i\}$ \cite{johansson2014regression} \\
    
    &  & $\RF(X)\pm q'_{n,\alpha}\{R_i\}$ \cite{zhang2019random}\\
    
    Ours (JaBS) & $\RF$ & $\RF(X)\pm q_{n,\alpha}\{R_i+\varepsilon\}$ (Theorem \ref{thm:PI_lower_bound})\\

    & & $\RF(X)\pm q_{n,\alpha}\{R_i-\varepsilon\}$ (Theorem \ref{thm:PI_upper_bound})\\

    & & $\RF(X)\pm q_{n,\alpha}\{R_i\}$ (Theorem \ref{thm:PI_asymptotic})\\
    \bottomrule
  \end{tabular}
\end{table}

\begin{table}
  \caption{Methods to construct prediction intervals using random forests: theoretical coverage}
  \label{table:compare2}
  \centering
  \begin{tabular}{lll}
    \toprule
    \cmidrule(r){1-3}
    Method     & Theoretical coverage     & Additional conditions \\
    \midrule
    J+ \cite{barber2021predictive} & $\geq 1-2\alpha$ & None   \\
    J+aB \cite{kim2020predictive}  & $\geq 1-2\alpha$ & Binomial $B$   \\
    JS \cite{barber2021predictive} & $\geq 1-\alpha-O(\sqrt{\nu})$  & Stability (algorithmic) \\
    J+aBS \cite{kim2020predictive} & $\geq 1-\alpha-O(\sqrt{\nu})$  & Stability (ensemble + algorithmic)\\
    JaB  & No guarantee \cite{johansson2014regression}  & - \\
      & $\to 1-\alpha$ \cite{zhang2019random}  &  Strong (additive model, consistency of RF predictor)\\
    Ours (JaBS) & $\geq 1-\alpha-O(\sqrt{\nu})$ & Stability (Theorem \ref{thm:PI_lower_bound})\\
    & $\leq 1-\alpha+\frac{1}{n+1}+O(\sqrt{\nu})$ & $+$ Distinct residuals (Theorem \ref{thm:PI_upper_bound})\\
    & $\to 1-\alpha$ & $+$ Uniformly equicontinuous CDF of $|Y-\RF(X)|$ \\
    & & and vanishing $\varepsilon,\nu$ (Theorem \ref{thm:PI_asymptotic})\\
    \bottomrule
  \end{tabular}
\end{table}

We compare these relevant methods to ours in Table \ref{table:compare} and Table \ref{table:compare2}, where the RF is set as the working algorithm for all methods and $(\varepsilon,\nu)$ is a general pair of stability parameters. We define $q_{n,\alpha}\{R_i\}$, $q^{+}_{n,\alpha}\{R_i\}$, $q^{-}_{n,\alpha}\{R_i\}$, and $q'_{n,\alpha}\{R_i\}$ as follows. Given $\{a_1,\ldots,a_n\}$, 
\begin{align*}
    q_{n,\alpha}\{a_i\}=q^{+}_{n,\alpha}\{a_i\}
    & {} \equiv
    \text{the $\lceil (1-\alpha)(n+1)\rceil$-th smallest value of $\{a_1,\ldots,a_n\}$,
    }\\
    q'_{n,\alpha}\{a_i\}    & {} \equiv
    \text{the $\lceil (1-\alpha)n\rceil$-th smallest value of $\{a_1,\ldots,a_n\}$,
    }\\
    q^{-}_{n,\alpha}\{a_i\}
    & {} \equiv
    \text{the $\lfloor \alpha(n+1)\rfloor$-th smallest value of $\{a_1,\ldots,a_n\}$,
    }
\end{align*}
where $\lfloor\cdot\rfloor$ is the floor function. Let $R_i^\mathrm{LOO} = |Y_i-\RF^{-i}(X_i)|$ be the LOO error, where $\RF^{-i}$ is trained without the $i$th training point, and by definition $\RF^{-i}\neq \RFi$. 

In Table \ref{table:compare}, we list the corresponding PI constructed from each method and the output predictors of each method. The number of output predictors directly reflects the computational cost. It is worth noting that acquiring the LOO predictor $\RF^{-i}$ needs substantial computation. In packages like \texttt{randomForest}, aggregating tree predictors to obtain the OOB predictor $\RFi$ also needs extra computation. However, the predicted value $\RFi(X_i)$ can be obtained immediately by calling the $\mathsf{predict()}$ function. The fact that the value of $\RFi(X)$ on a test point is \emph{not} needed further reduces the computational cost of JaB and our method, which only need one output RF predictor, and are more favorable computationally.  

In Table \ref{table:compare2}, we list the coverage of the PI constructed from each method, as well as the additional conditions (beyond iid data) needed to achieve the coverage. Note that J+ does not require any additional conditions to achieve the coverage lower bound $1-2\alpha$, but J+aB requires that the number of trees $B$ be a Binomial random variable. For JS, J+aBS, and our method, stability is needed to achieve the coverage lower bound $1-\alpha-O(\sqrt{\nu})$. With additional mild assumptions, the coverage upper bound and asymptotic coverage of our method can be established. However, there is no guarantee of coverage for JaB without strong assumptions.   

In summary, our theoretical work provides a series of coverage guarantees to a computationally feasible method for constructing PIs based on the RF algorithm. In the following, we will establish the lower and upper bound of coverage, as well as the asymptotic coverage.

\subsection{Non-asymptotic coverage guarantees}
\begin{theorem}[Coverage lower bound]\label{thm:PI_lower_bound}
    Suppose the RF predictor $\RF$ satisfies the stability condition as in Theorem \ref{thm:RF_stability}. Then we have for a test point $(X,Y)$ that
    \begin{align}\label{eq:PI_lower_bound}
        \p(Y\in\RF(X)\pm q_{n,\alpha}\{R_i+\varepsilon_{n,B}\}) \geq 1-\alpha - \nu_1-2\sqrt{\nu_2}-2\sqrt{\nu_3}.
    \end{align}
\end{theorem}
This result is established by starting from the analysis of an imaginary extended dataset $\overline D = D\cup\{(X,Y)\}$, where the test point is \emph{assumed} to be known. We denote $(X,Y)$ as $(X_{n+1},Y_{n+1})$ for convenience. For all points in $\overline D$, consider the derandomized RF predictor $\widetilde \rfall^{\backslash i}$ that is built on $n$ data points without the $i$th point in $\overline D$, $i\in[n+1]$. One can then define the OOB error $\widetilde r_i\equiv |Y_i-\widetilde \rfall^{\backslash i}|$. Since all data are iid, we have that $\p(\widetilde r_{n+1}\leq q_{n,\alpha}\{\widetilde r_i\})\geq 1-\alpha$, where $q_{n,\alpha}\{\widetilde r_i\}$ is 
the $\lceil (1-\alpha)(n+1)\rceil$-th smallest value of $\{\widetilde r_1,\ldots,\widetilde r_n\}$. Next, notice $\widetilde r_{n+1}=|Y_{n+1}-\widetilde \rfall^{\backslash (n+1)}(X_{n+1})|=|Y_{n+1}-\rfall(X_{n+1})|$ by the definitions of $\widetilde \rfall^{\backslash (n+1)}$ and $\rfall$. By concentration of measure, $\rfall(X_{n+1})$ can be approximated by $\RF(X_{n+1})$, and thus $\widetilde r_{n+1}$ can be roughly replaced with $|Y_{n+1}-\RF(X_{n+1})|$, which is desired. Then by stability of $\rfall$, $\{\widetilde r_i\}$ can be approximated by $\{r_i\equiv|Y_i-\rfi(X_i)|\}$. Although $\{r_i\}$ is still unavailable in practice, by applying the idea of concentration of measure again, $\{r_i\}$ can be further approximated by $\{R_i\}$, which is accessible given $D$. Eventually, we can bound $|Y_{n+1}-\RF(X_{n+1})|$ in terms of $\{R_i\}$. The approximations are accounted for by the stability parameters in Theorem \ref{thm:RF_stability}. 

If we further assume that there are no ties among $\{\widetilde r_i\}, i\in[n+1]$, a typical case when $Y$ is continuous, then we can also establish the upper bound of coverage.
\begin{theorem}[Coverage upper bound]\label{thm:PI_upper_bound}
    Suppose there are no ties in $\{\widetilde r_i\}, i\in[n+1]$, and the RF predictor $\RF$ satisfies the stability condition as in Theorem \ref{thm:RF_stability}. Then
    \begin{align}\label{eq:PI_upper_bound}
        \p(Y\in\RF(X)\pm q_{n,\alpha}\{R_i-\varepsilon_{n,B}\}) \leq 1-\alpha + \frac{1}{n+1} + \nu_1 + 2\sqrt{\nu_2} + 2\sqrt{\nu_3}.
    \end{align}
\end{theorem}
The upper bound can be established because if there are no ties among $\widetilde r_1,\ldots,\widetilde r_{n+1}$, then $\p(\widetilde r_{n+1}\leq q_{n,\alpha}\{\widetilde r_i\})\leq 1-\alpha+\frac{1}{n+1}$. The apparent symmetry between the lower and upper bound originates from the fact that they both are established by using the RF stability once and the concentration of measure twice. Note that this idea can be applied to JS intervals for an arbitrary stable algorithm in exactly the same way, providing a complement to the lower bound for JS intervals established in \cite{barber2021predictive}.
\begin{corollary}[Coverage upper bound for  jackknife-with-stability intervals]
\label{cor:general_stable_upper_bound}
    Let $\hat f$ be a predictor trained on $n$ iid data points and $\hat f^{- i}$ be the LOO predictor without the $i$th point. Suppose $\hat f$ is stable with $\p(|\hat f(X)-\hat f^{- i}(X)|>\varepsilon)\leq \nu$, and the LOO errors are distinct on the extended training set that includes an iid test point $(X,Y)$. Then we have $\p\left(|Y-\hat f(X)|\leq q_{n,\alpha}\{ r_i-\varepsilon\}\right)\leq 1-\alpha +\frac{1}{n+1}+2\sqrt{\nu}$, where $r_i$ are the LOO errors on the original training set.
\end{corollary}

\subsection{Asymptotic coverage guarantee}
As shown above, the stability parameters $(\varepsilon_{n,B},\nu_{n,B})$ can vanish when $n\to\infty$. It is reasonable to expect that  $\p(Y\in\RF(X)\pm q_{n,\alpha}\{R_i\})\to 1-\alpha$ in this limit, as is consistent with numerous empirical observations \cite{johansson2014regression,zhang2019random}. However, to achieve this goal, it seems that more assumptions are unavoidable. In \cite{zhang2019random}, the guaranteed coverage of the JaB method is established by assuming that $\RF(X)$ converges to some $f_0(X)$ in probability as $n\to\infty$, where $f_0$ is the true regression function of an additive model that generates the data. We show that this can be done under weaker conditions. 

\begin{theorem}[Asymptotic coverage]\label{thm:PI_asymptotic}  Denote $F_n$ as the CDF of $|Y-\RF(X)|$. Suppose $\{F_n\}_{n\geq n_0}$ is uniformly equicontinuous for some $n_0$. Then $\p(Y\in\RF(X)\pm q_{n,\alpha}\{R_i\})\to1-\alpha$ as $n\to\infty$ when conditions in Theorem \ref{thm:PI_lower_bound}, Theorem \ref{thm:PI_upper_bound}, and Corollary \ref{cor:epsilon_nu_to_0} are satisfied.
\end{theorem}
\begin{remark}
   Intuitively, using errors from $\RFi$ that are trained on $n-1$ points to approximate those from $\RF$, trained on $n$ points, we only need this approximation to be exact asymptotically. There is no need for $\RF$ itself to converge to anything. This is one major conceptual difference between our work and \cite{zhang2019random}, and it is in this sense that our assumption is weaker. Practically, this kind of PI is recommended as it does not involve $(\varepsilon_{n,B},\nu_{n,B})$, and has great performance on numerous datasets.   
\end{remark}

\section{Conclusion}
In this work, for the first time, we theoretically establish the stability property of the greedy version of random forests, which is implemented in popular packages. The theoretical guarantee is based on a light-tail assumption of the marginal distribution of the squared response $Y^2$. However, numerical evidence suggests that this stability could persist in much broader situations. Based on the stability property and some mild conditions, we also establish finite-sample lower and upper bounds of coverage, as well as the exact coverage asymptotically, for prediction intervals constructed from the out-of-bag error of random forests, justifying random forests as an appealing method to provide both point and interval prediction simultaneously.  

\begin{ack}
Much of the work was completed while Yan Wang was a PhD student in the Department of Statistics at Iowa State University. This research was supported in part by the US National Science Foundation under grant HDR:TRIPODS 19-34884.
\end{ack}

\bibliographystyle{plainnat}
\bibliography{ref}

\begin{thebibliography}{35}
\providecommand{\natexlab}[1]{#1}
\providecommand{\url}[1]{\texttt{#1}}
\expandafter\ifx\csname urlstyle\endcsname\relax
  \providecommand{\doi}[1]{doi: #1}\else
  \providecommand{\doi}{doi: \begingroup \urlstyle{rm}\Url}\fi

\bibitem[Angelopoulos et~al.(2023)Angelopoulos, Bates, et~al.]{angelopoulos2023conformal}
Anastasios~N Angelopoulos, Stephen Bates, et~al.
\newblock Conformal prediction: A gentle introduction.
\newblock \emph{Foundations and Trends{\textregistered} in Machine Learning}, 16\penalty0 (4):\penalty0 494--591, 2023.

\bibitem[Asuncion and Newman(2007)]{asuncion2007uci}
Arthur Asuncion and David Newman.
\newblock {UCI} {M}achine {L}earning {R}epository, 2007.

\bibitem[Barber et~al.(2021)Barber, Candes, Ramdas, and Tibshirani]{barber2021predictive}
Rina~Foygel Barber, Emmanuel~J Candes, Aaditya Ramdas, and Ryan~J Tibshirani.
\newblock Predictive inference with the jackknife+.
\newblock \emph{The Annals of Statistics}, 49\penalty0 (1):\penalty0 486--507, 2021.

\bibitem[Bertsimas et~al.(2022)Bertsimas, Dunn, and Paskov]{bertsimas2022stable}
Dimitris Bertsimas, Jack Dunn, and Ivan Paskov.
\newblock Stable classification.
\newblock \emph{The Journal of Machine Learning Research}, 23\penalty0 (1):\penalty0 13401--13453, 2022.

\bibitem[Biau(2012)]{biau2012analysis}
G{\'e}rard Biau.
\newblock Analysis of a random forests model.
\newblock \emph{The Journal of Machine Learning Research}, 13\penalty0 (1):\penalty0 1063--1095, 2012.

\bibitem[Biau and Scornet(2016)]{biau2016random}
G{\'e}rard Biau and Erwan Scornet.
\newblock A random forest guided tour.
\newblock \emph{Test}, 25:\penalty0 197--227, 2016.

\bibitem[Biau et~al.(2008)Biau, Devroye, and Lugosi]{biau2008consistency}
G{\'e}rard Biau, Luc Devroye, and G{\"a}bor Lugosi.
\newblock Consistency of random forests and other averaging classifiers.
\newblock \emph{The Journal of Machine Learning Research}, 9\penalty0 (9):\penalty0 2015--2033, 2008.

\bibitem[Boucheron et~al.(2013)Boucheron, Lugosi, and Massart]{boucheron2013concentration}
St{\'e}phane Boucheron, G{\'a}bor Lugosi, and Pascal Massart.
\newblock \emph{Concentration Inequalities: A Nonasymptotic Theory of Independence}.
\newblock Oxford University Press, 2013.

\bibitem[Bousquet and Elisseeff(2002)]{bousquet2002stability}
Olivier Bousquet and Andr{\'e} Elisseeff.
\newblock Stability and generalization.
\newblock \emph{The Journal of Machine Learning Research}, 2:\penalty0 499--526, 2002.

\bibitem[Breiman(2001)]{breiman2001random}
Leo Breiman.
\newblock Random forests.
\newblock \emph{Machine Learning}, 45\penalty0 (1):\penalty0 5--32, 2001.

\bibitem[Brooks et~al.(1989)Brooks, Pope, and Marcolini]{brooks1989airfoil}
Thomas~F Brooks, D~Stuart Pope, and Michael~A Marcolini.
\newblock \emph{Airfoil Self-Noise and Prediction}.
\newblock National Aeronautics and Space Administration, Office of Management, 1989.

\bibitem[Coraddu et~al.(2016)Coraddu, Oneto, Ghio, Savio, Anguita, and Figari]{coraddu2016machine}
Andrea Coraddu, Luca Oneto, Aessandro Ghio, Stefano Savio, Davide Anguita, and Massimo Figari.
\newblock Machine learning approaches for improving condition-based maintenance of naval propulsion plants.
\newblock \emph{Proceedings of the Institution of Mechanical Engineers, Part M: Journal of Engineering for the Maritime Environment}, 230\penalty0 (1):\penalty0 136--153, 2016.

\bibitem[Correa and Romero(2021)]{correa2021asymptotic}
Jos{\'e}~R Correa and Mat{\'\i}as Romero.
\newblock On the asymptotic behavior of the expectation of the maximum of iid random variables.
\newblock \emph{Operations Research Letters}, 49\penalty0 (5):\penalty0 785--786, 2021.

\bibitem[Denil et~al.(2014)Denil, Matheson, and De~Freitas]{denil2014narrowing}
Misha Denil, David Matheson, and Nando De~Freitas.
\newblock Narrowing the gap: Random forests in theory and in practice.
\newblock In \emph{International Conference on Machine Learning}, pages 665--673. PMLR, 2014.

\bibitem[Downey(1990)]{downey1990distribution}
Peter~J Downey.
\newblock Distribution-free bounds on the expectation of the maximum with scheduling applications.
\newblock \emph{Operations Research Letters}, 9\penalty0 (3):\penalty0 189--201, 1990.

\bibitem[Grisoni et~al.(2016)Grisoni, Consonni, Vighi, Villa, and Todeschini]{grisoni2016investigating}
Francesca Grisoni, Viviana Consonni, Marco Vighi, Sara Villa, and Roberto Todeschini.
\newblock Investigating the mechanisms of bioconcentration through {QSAR} classification trees.
\newblock \emph{Environment International}, 88:\penalty0 198--205, 2016.

\bibitem[Johansson et~al.(2014)Johansson, Bostr{\"o}m, L{\"o}fstr{\"o}m, and Linusson]{johansson2014regression}
Ulf Johansson, Henrik Bostr{\"o}m, Tuve L{\"o}fstr{\"o}m, and Henrik Linusson.
\newblock Regression conformal prediction with random forests.
\newblock \emph{Machine Learning}, 97:\penalty0 155--176, 2014.

\bibitem[Kim et~al.(2020)Kim, Xu, and Barber]{kim2020predictive}
Byol Kim, Chen Xu, and Rina Barber.
\newblock Predictive inference is free with the jackknife+-after-bootstrap.
\newblock \emph{Advances in Neural Information Processing Systems}, 33:\penalty0 4138--4149, 2020.

\bibitem[Kontorovich(2014)]{kontorovich2014concentration}
Aryeh Kontorovich.
\newblock Concentration in unbounded metric spaces and algorithmic stability.
\newblock In \emph{International Conference on Machine Learning}, pages 28--36. PMLR, 2014.

\bibitem[Lei et~al.(2018)Lei, G'Sell, Rinaldo, Tibshirani, and Wasserman]{lei2018distribution}
Jing Lei, Max G'Sell, Alessandro Rinaldo, Ryan~J Tibshirani, and Larry Wasserman.
\newblock Distribution-free predictive inference for regression.
\newblock \emph{Journal of the American Statistical Association}, 113\penalty0 (523):\penalty0 1094--1111, 2018.

\bibitem[Liaw and Wiener(2002)]{liaw2002classification}
Andy Liaw and Matthew Wiener.
\newblock Classification and regression by random{F}orest.
\newblock \emph{R News}, 2\penalty0 (3):\penalty0 18--22, 2002.

\bibitem[Lin and Jeon(2006)]{lin2006random}
Yi~Lin and Yongho Jeon.
\newblock Random forests and adaptive nearest neighbors.
\newblock \emph{Journal of the American Statistical Association}, 101\penalty0 (474):\penalty0 578--590, 2006.

\bibitem[Maurer and Pontil(2021)]{maurer2021concentration}
Andreas Maurer and Massimiliano Pontil.
\newblock Concentration inequalities under sub-{G}aussian and sub-exponential conditions.
\newblock \emph{Advances in Neural Information Processing Systems}, 34:\penalty0 7588--7597, 2021.

\bibitem[Papadopoulos et~al.(2002)Papadopoulos, Proedrou, Vovk, and Gammerman]{papadopoulos2002inductive}
Harris Papadopoulos, Kostas Proedrou, Volodya Vovk, and Alex Gammerman.
\newblock Inductive confidence machines for regression.
\newblock In \emph{European Conference on Machine Learning}, pages 345--356. Springer, 2002.

\bibitem[Scornet(2016)]{scornet2016random}
Erwan Scornet.
\newblock Random forests and kernel methods.
\newblock \emph{IEEE Transactions on Information Theory}, 62\penalty0 (3):\penalty0 1485--1500, 2016.

\bibitem[Scornet et~al.(2015)Scornet, Biau, and Vert]{scornet2015consistency}
Erwan Scornet, G{\'e}rard Biau, and Jean-Philippe Vert.
\newblock Consistency of random forests.
\newblock \emph{The Annals of Statistics}, 43\penalty0 (4):\penalty0 1716--1741, 2015.

\bibitem[Shafer and Vovk(2008)]{shafer2008tutorial}
Glenn Shafer and Vladimir Vovk.
\newblock A tutorial on conformal prediction.
\newblock \emph{Journal of Machine Learning Research}, 9\penalty0 (3), 2008.

\bibitem[Soloff et~al.(2023)Soloff, Barber, and Willett]{soloff2023bagging}
Jake~A Soloff, Rina~Foygel Barber, and Rebecca Willett.
\newblock Bagging provides assumption-free stability.
\newblock \emph{arXiv preprint arXiv:2301.12600}, 2023.

\bibitem[Steinberger and Leeb(2023)]{steinberger2023conditional}
Lukas Steinberger and Hannes Leeb.
\newblock Conditional predictive inference for stable algorithms.
\newblock \emph{The Annals of Statistics}, 51\penalty0 (1):\penalty0 290--311, 2023.

\bibitem[Tang et~al.(2018)Tang, Garreau, and von Luxburg]{tang2018random}
Cheng Tang, Damien Garreau, and Ulrike von Luxburg.
\newblock When do random forests fail?
\newblock \emph{Advances in Neural Information Processing Systems}, 31, 2018.

\bibitem[Vovk(2012)]{vovk2012conditional}
Vladimir Vovk.
\newblock Conditional validity of inductive conformal predictors.
\newblock In \emph{Asian Conference on Machine Learning}, pages 475--490. PMLR, 2012.

\bibitem[Vovk(2015)]{vovk2015cross}
Vladimir Vovk.
\newblock Cross-conformal predictors.
\newblock \emph{Annals of Mathematics and Artificial Intelligence}, 74\penalty0 (1):\penalty0 9--28, 2015.

\bibitem[Vovk et~al.(2005)Vovk, Gammerman, and Shafer]{vovk2005algorithmic}
Vladimir Vovk, Alexander Gammerman, and Glenn Shafer.
\newblock \emph{Algorithmic Learning in a Random World}.
\newblock Springer Science \& Business Media, 2005.

\bibitem[Yeh(1998)]{yeh1998modeling}
I-C Yeh.
\newblock Modeling of strength of high-performance concrete using artificial neural networks.
\newblock \emph{Cement and Concrete Research}, 28\penalty0 (12):\penalty0 1797--1808, 1998.

\bibitem[Zhang et~al.(2020)Zhang, Zimmerman, Nettleton, and Nordman]{zhang2019random}
Haozhe Zhang, Joshua Zimmerman, Dan Nettleton, and Daniel~J Nordman.
\newblock Random forest prediction intervals.
\newblock \emph{The American Statistician}, 74\penalty0 (4):\penalty0 392--406, 2020.

\end{thebibliography}

\newpage
\appendix
\section{Proof of Corollary \ref{cor:rf_consistency}}
Let  
\begin{align}\label{eq:eta}
    \eta\equiv\frac{p}{1-p}+\frac{q}{(1-p)^2}.
\end{align}
Then $\delta(D,X)=Z_{(n)}^2\eta/(\varepsilon^2n)$. Also, by (\ref{eq:p}) and (\ref{eq:q}), $\eta$ is upper bounded as 
\begin{align*}
\eta &{} = \frac{1}{1-p} -1 + \frac{\left(1-\frac{1}{n}\right)^{2n}-\left(1-\frac{2}{n}\right)^n}{\left(1-\frac{1}{n}\right)^{2n}}\\
&{} = \frac{1}{\left(1-\frac{1}{n}\right)^{n}} - \frac{\left(1-\frac{2}{n}\right)^n}{\left(1-\frac{1}{n}\right)^{2n}} = \frac{\left(1-\frac{1}{n}\right)^{n}-\left(1-\frac{2}{n}\right)^{n}}{\left(1-\frac{1}{n}\right)^{2n}} \\
&{} \leq \frac{\frac{1}{e}-0}{1/16} \text{   (since $n\geq2$ and $(1-1/n)^n$ is monotonically increasing in $n$)}\\
&{} = \frac{16}{e} < 6.
\end{align*}
As a result, for all $n\geq 2$, we have $\eta/n<3$.

Then, note that for a non-negative random variable $W$, it holds that $\E[W]=\int_0^\infty \p(W>t)dt$. By (\ref{eq:RF_range}) and (\ref{eq:rf_conditional_stability}), we have
\begin{align*}       \E\left[\left|\rfall(X)-\rfi(X)\right|\right] 
& {} = \E_{D,X}\left[\int_0^\infty \p_{\xi,r|D,X}\left(\left|\rfall(X)-\rfi(X)\right|>\varepsilon
\middle |D,X\right)d\varepsilon\right]\\
& {} = \E_{D,X}\left[\int_0^{2Z_{(n)}} \p_{\xi,r|D,X}\left(\left|\rfall(X)-\rfi(X)\right|>\varepsilon\middle |D,X\right)d\varepsilon\right]\\
& {} \leq \E_{D,X}\left[ \int_0^{2Z_{(n)}} \min\{\delta(D,X),1\}d\varepsilon\right]\\
& {} = \E_{D,X}\left[
\int_0^{\sqrt{Z_{(n)}^2\eta/n}} 1d\varepsilon + \int_{\sqrt{Z_{(n)}^2\eta/n}}^{2Z_{(n)}} \frac{Z_{(n)}^2\eta}{\varepsilon^2 n}d\varepsilon
\right]\\
&{} \text{(note $2Z_{(n)}> \sqrt{Z_{(n)}^2\eta/n}$ for $n\geq 2$)}\\
&{} < \E\left[2\sqrt{Z_{(n)}^2\eta/n}\right]\\
&{} \leq 2\sqrt{\E[Z_{(n)}^2]\eta/n}\text{ (by Jensen's inequality)},
\end{align*}
which completes the first part of the corollary. For the second part, we need the following lemma, which is established in \cite{downey1990distribution,correa2021asymptotic}.
\begin{lemma}\label{lem:E_max}
    For iid random variables $W_i,i\in[n]$, if $\E[|W_i|]<\infty$, then $\lim_{n\to \infty}\E[W_{(n)}]/n = 0$.
\end{lemma}
By this lemma, $\E[Y^2]<\infty$ implies $\E[Z_{(n)}^2]/n\to 0$ as $n\to\infty$. Since $\eta$ is bounded, it is clear that $|\rfall(X)-\rfi(X)|$ converges to 0 in mean, which further implies convergence in probability, completing the proof.

\begin{remark}
    We can actually have a slightly tighter bound in (\ref{eq:rf_conditional_stability}), and thus in Theorem \ref{thm:rf_stability} and Corollary \ref{cor:rf_consistency}. A closer look at the method in \cite{soloff2023bagging} indicates that the key point is to bound the conditional variance of $\tree$, $\E_{\xi,r|D,X}[(\rfall-\tree)^2]$. Note that $\rfall=p\times\E_{\xi,r|D,X}[\tree|i\in r]+(1-p)\times\E_{\xi,r|D,X}[\tree|i\notin r]\equiv p\times \rfall^i + (1-p)\times \rfi$, but $|\rfall^i - \rfi|$ is also bounded by $2Z_{(n)}$, and we have $|\rfall(X)-\rfi(X)|\leq 2p\times Z_{(n)}$. By (\ref{eq:p}), $p$ converges to a constant as $n\to\infty$. Hence this tighter bound does not provide any qualitative difference as $n$ increases, and we ignore this minor improvement in this work. 
\end{remark}

\section{Proof of Theorem \ref{thm:RF_conditional_stability} and discussion}
By the triangle inequality and the union bound, we have for some $t>0$ that
\begin{align}\label{eq:T123}
     \frac{1}{n}\sum_{i=1}^n\p_{\bm{\xi,r}|D,X}
    & 
     \left(
     \left.
    \left|\RF(X)-\RFi(X)\right|>\varepsilon + 2t
    \middle |D,X
    \right.
    \right) \notag\\
    \leq {} &  
    \frac{1}{n}\sum_{i=1}^n\p_{\bm{\xi,r}|D,X}\left(
    \left|\RF(X)-\rfall(X)\right|> t |D,X
    \right) \notag\\
    & + \frac{1}{n}\sum_{i=1}^n\p_{\bm{\xi,r}|D,X}\left(
    \left.
    \left|\rfall(X)-\rfi(X)\right|> \varepsilon
    \middle |D,X
    \right.
    \right) \notag\\
    & + \frac{1}{n}\sum_{i=1}^n\p_{\bm{\xi,r}|D,X}\left(
    \left.
    \left|\rfi(X)-\RFi(X)\right|> t
    \middle |D,X
    \right.
    \right) \notag\\
    \equiv {} & T_1 + T_2 + T_3.
\end{align}
In $T_2$, $\rfall$ and $\rfi$ are derandomized predictors, and by (\ref{eq:rf_conditional_stability}), we have 
\begin{align*}
T_2 \leq \delta.
\end{align*}
Of course, this result is nontrivial when $\delta\in(0,1)$. Note that $\rfall$ is the RF with an infinite number of trees, where each tree is characterized by a pair of independent random variables $(\xi,r)$, while $\RF$ consists of a finite number of $B$ trees. Since the value of $\RF(X)$ is bounded in $\left[-Z_{(n)},Z_{(n)}\right]$, $T_1$ can be bounded by Hoeffding's inequality as
\[
T_1\leq 2\exp\left(-\frac{Bt^2}{2Z_{(n)}^2}\right).
\]
The analysis of the first two terms $T_1$ and $T_2$ is similar to that in \cite{soloff2023bagging}, but the bound for $T_3$ to be developed is specific to the RF predictor, and is nontrivially different from the case in \cite{soloff2023bagging} where the number of bags for each LOO predictor is fixed to $B$. In our setting, $\RFi(X)\in[-Z_{(n)},Z_{(n)}]$ is the aggregation of $B_i$ trees where the $i$th training point is not included. Hence by Hoeffding's inequality and (\ref{eq:BiMGF}), we have
\begin{align*}
T_3\leq \frac{1}{n}\sum_{i=1}^n 
\E_{B_i|D,X}\left[
\left.
2\exp\left(-\frac{B_it^2}{2Z_{(n)}^2}\right)
\middle |D,X
\right.
\right]
= 2\left(p+(1-p)e^{-t^2/(2Z_{(n)}^2)}\right)^B.
\end{align*}

Now, we choose $t$ such that $(0,1)\ni \delta = \exp\left(-Bt^2/(2Z_{(n)}^2)\right)$, which then yields $t = \sqrt{\frac{2Z_{(n)}^2}{B}\ln\left(\frac{1}{\delta}\right)}$, and thus
\begin{align}\label{eq:g_def}
T_3 \leq 2\left(p+(1-p)\delta^{\frac{1}{B}}\right)^B\equiv g(p,\delta,B).
\end{align}
Combining these results together completes the proof.

\paragraph{Discussion.} Note that $g(p,\delta,B)$ is monotonically decreasing in $B$, because
\begin{align*}
    \frac{\partial \ln [g(p,\delta,B)/2]}{\partial B} 
    &{} = \ln\left(p+(1-p)\delta^{1/B}\right) - \frac{B(1-p)\delta^{1/B}}{p+(1-p)\delta^{1/B}}\frac{\ln\delta}{B^2}\\
    &{} \leq p\ln 1+ (1-p)\ln\delta^{1/B} - \frac{(1-p)\delta^{1/B}}{p+(1-p)\delta^{1/B}}\frac{\ln\delta}{B} \text{ (by Jensen's inequality)}\\
    &{} = (1-p)\left[1-\frac{\delta^{1/B}}{p+(1-p)\delta^{1/B}}\right]\frac{\ln\delta}{B}\\
    &{} =\frac{1-p}{B}\frac{p(1-\delta^{1/B})}{p+(1-p)\delta^{1/B}}\ln\delta\\
    &{} < 0 \text{ (as $0<\delta<1$).} 
\end{align*}
Moreover, since $g\geq 0$ and monotonically decreasing in $B$, it has a limit when $p$ and $\delta$ are fixed and $B\to\infty$:
\begin{align}\label{eq:g_limit}    \lim_{B\to\infty}\ln[g(p,\delta,B)/2] 
    &{} = \lim_{B\to\infty} B\ln\left[p+(1-p)\delta^{1/B}\right]\notag\\
    &{} = \lim_{B\to\infty} B\ln\left[p+(1-p)e^{\ln\delta/B}\right]\notag\\
    &{} = \lim_{B\to\infty} B\ln\left[p+(1-p)\left(1+\frac{\ln\delta}{B}\right)+o\left(\frac{\ln\delta}{B}\right)\right]\notag\\
    &{} = \lim_{B\to\infty} B\ln\left[1+\frac{\ln \delta^{1-p}}{B}+o\left(\frac{\ln\delta}{B}\right)\right]\notag\\
    &{} = \lim_{B\to\infty} \ln\left[\left(1+\frac{\ln \delta^{1-p}}{B}+o\left(\frac{\ln\delta}{B}\right)\right)^B\right]\notag\\
    &{} = \ln\delta^{1-p}.
\end{align}
Hence we have that $\lim_{B\to\infty}g(p,\delta,B) = 2\delta^{1-p}$. This result implies that if $B$ is allowed to increase in a way that is independent of $n$, then under the conditions of Theorem \ref{thm:RF_conditional_stability}, we have
\begin{align*}
\lim_{B\to\infty} \frac{1}{n}\sum_{i=1}^n\p_{\bm{\xi,r}|D,X}
     \left(
     \left.
    \left|\RF(X)-\RFi(X)\right|>\varepsilon + 2\sqrt{\frac{2Z_{(n)}^2}{B}\ln\left(\frac{1}{\delta}\right)}
    \middle | D,X
    \right.
    \right)  
    \leq 5\delta^{1-p}.
\end{align*}
We will come back to this analysis in the proof of Corollary \ref{cor:epsilon_nu_to_0} below where $B$ grows with $n$, but it still holds that $\ln\delta(n) = o(B(n))$.
\section{Proofs of Theorem \ref{thm:RF_stability}, Lemma \ref{lem:sub-gamma_tail}, and Corollary \ref{cor:epsilon_nu_to_0}; and discussion}
\subsection{Proof of Theorem \ref{thm:RF_stability}}
Similar to the proof for conditional stability, we have for some $\varepsilon_2\equiv\varepsilon, \varepsilon_1=\varepsilon_3\equiv t>0$ that 
    \begin{align*}
    \p_{D,X,\bm{\xi,r}}
    & 
     \left(
    \left|\RF(X)-\RFi(X)\right|>\varepsilon + 2t
    \right) \notag\\
    \leq {} &  
    \p_{D,X,\bm{\xi,r}}\left(
    \left|\RF(X)-\rfall(X)\right|> t
    \right) \notag\\
    & + \p_{D,X }\left(
    \left|\rfall(X)-\rfi(X)\right|> \varepsilon
    \right) \notag\\
    & + \p_{D,X,\bm{\xi,r}}\left(
    \left|\rfi(X)-\RFi(X)\right|> t
    \right) \notag\\
    \equiv {} & T'_1 + T'_2 + T'_3.
\end{align*}
We first consider the $T_2'$ term, which can be bounded when the stability condition (\ref{eq:rf_stability}) for $\rfall$  is applied. Thus  
\begin{align}\label{eq:RF_nu}
T_2' \leq \nu \equiv \nu_2/\lambda.
\end{align}
Then, taking expectation with respect to data distribution for $T_1$ and $T_3$ in (\ref{eq:T123}) yields 
\[
T'_1 \leq \E\left[2\exp\left(-\frac{Bt^2}{2Z_{(n)}^2}\right)\right],
\]
and
\[
T_3'\leq \E\left[2\left(p+(1-p)e^{-t^2/(2Z_{(n)}^2)}\right)^B\right].
\]
In order to obtain more informative bounds than the above ones, we introduce the following lemma.
\begin{lemma}
    Let $W$ be a random variable with finite mean, and suppose that $h$ is a generic monotonically increasing function of $W$ and is bounded in $[0,1]$. Then we have for some $\lambda > 1$ that
    \[
    \E[h(W)]\leq h(\lambda\E[W]) + \p(W>\lambda\E[W]).
    \]
\end{lemma}
\begin{proof}
    This can be established by noting that
    \begin{align*}
    \E[h(W)] 
    &{}= \E[h(W)\mathbb{I}\{W \leq \lambda \E[W]\}] + \E[h(W)\mathbb{I}\{W > \lambda \E[W]\}]\\
    &{} \leq h(\lambda\E[W]) + \p(W > \lambda \E[W]).
    \end{align*}
    The first term on the right-hand side results from $h$ being monotonically increasing, and the second term from $h$ being upper bounded by $1$.
\end{proof}

Applying this lemma to the upper bounds of $T'_1$ and $T'_3$ above yields 
\[
T'_1 \leq 2\exp\left(-\frac{Bt^2}{2\lambda\E[Z_{(n)}^2]}\right) + 2\p\left( Z_{(n)}^2 > \lambda\E[Z_{(n)}^2]\right)\equiv \nu_1,
\]
and
\[
T_3'\leq 2\left(p+(1-p)e^{-t^2/(2\lambda\E[Z_{(n)}^2])}\right)^B + 2\p\left( Z_{(n)}^2 > \lambda\E[Z_{(n)}^2]\right)\equiv\nu_3.
\]
Now we choose $t$ such that $\nu_2 = \exp\left(-Bt^2/(2\lambda\E[Z_{(n)}^2])\right)$ to complete the proof of Theorem \ref{thm:RF_stability}.
\begin{remark}
    Unlike the conditional stability case, an additional tail probability term is introduced in the upper bounds for both $T'_1$ and $T'_3$. This is because when $Z_{(n)}^2$ is unbounded and no detailed information of the data-generating distribution is provided, there seems no universal way to control terms like $\E[\exp(-1/Z_{(n)}^2)]$. We tackle this problem by introducing the tail probability. As shown below, it is not difficult to control this term for typical data distributions. 
\end{remark}

\subsection{Proof of Lemma \ref{lem:sub-gamma_tail}}
Suppose $Z_1^2$ is sub-gamma on the right tail with parameters $(\sigma^2, c)$, then it is known that $\E[Z_{(n)}^2]$ is at most on the order of $O(\ln n)$ by the maximal inequality for sub-gamma random variables \cite{boucheron2013concentration}. More precisely, we have
\[
\E[Z_{(n)}^2] \leq \sqrt{2\sigma^2\ln n} + c\ln n.
\]
When $c=0$, this reduces to the maximal inequality for sub-Gaussian random variables. When $c>0$, the leading term of the bound is $O(\ln n)$, meaning that the tail probability decays in an exponential, rather than Gaussian, way. The parameter $c$ also sets an upper bound quantitatively, and for large $n$, we must have $\E[Z_{(n)}^2]\sim a \ln n$ with $a\leq c$.

Since data are iid, the tail probability for the order statistic $Z_{(n)}^2$ can be rewritten as
\begin{align*}
\p\left( Z_{(n)}^2 > \lambda\E[Z_{(n)}^2]\right) 
& = {}  1 - \left[1- \p\left( Z_1^2  > \lambda \E[Z_{(n)}^2]\right)\right]^n \leq n \p\left( Z_1^2  > \lambda \E[Z_{(n)}^2]\right). 
\end{align*}
Moreover, we have the concentration inequality that for every $t>0$ \cite[][]{boucheron2013concentration},
\[
\p\left(Z_1^2>\sqrt{2\sigma^2 t } + ct + \E[Z_1^2]\right) \leq e^{-t}.
\]
Setting $\lambda \E[Z_{(n)}^2]=\sqrt{2\sigma^2 t } + ct + \E[Z_1^2]$, and noticing $t = \frac{\lambda a}{c}\ln n + o(\ln n)$ for large $n$, we obtain for $\lambda>\frac{c}{a}$ that
\begin{align*}
    \lim_{n\to\infty} \p\left(Z_{(n)}^2>\lambda\E\left[Z_{n}^2\right]\right) 
    \leq \lim_{n\to\infty} ne^{-\frac{\lambda a}{c}\ln n}
    = \lim_{n\to\infty}\left(\frac{1}{n}\right)^{\frac{\lambda a}{c}-1} = 0,
\end{align*}
completing the proof.

\subsection{Proof of Corollary \ref{cor:epsilon_nu_to_0}}
Consider each pair of $(\varepsilon_i,\nu_i$) defined above. First, the $(\varepsilon_2, \nu_2/\lambda)$ pair satisfies the stability condition (\ref{eq:rf_stability}) of the derandomized RF. Hence for $(\varepsilon_2,\nu_2)$ to converge to $0$, we require
\[\varepsilon_2(n)\to 0 \text{ as } n\to\infty,\]
and 
\[\nu_2(n)=\frac{\lambda\E[Z_{(n)}^2]\eta(n)}{n\varepsilon_2^2(n)}\to 0 \text{ as } n\to\infty,\]
where $\eta(n)$ is defined in (\ref{eq:eta}). It is clear that $\varepsilon_2(n)=o(1)$ is required. Since $\eta(n)$ converges to a positive constant, and $\E[Z_{(n)}^2]\sim a\ln n$ with $a\leq c$ by the sub-gamma assumption \footnote{Strictly speaking, it is possible for a sub-gamma random variable to have the scaling of $\E[Z_{(n)}^2]$ in between $\sqrt{\ln n}$ and $\ln n$, but the point is that the heavier the tail (i.e., the faster growth of $\E[Z_{(n)}^2]$), the more difficult to achieve stability in theory. So we focus on the most heavy-tail case that $\E[Z_{(n)}^2]$ scales as $\ln n$.}, we thus have
\[\nu_2(n)=\Theta\left(\frac{\ln n}{n\varepsilon_2^2(n)}\right).\]
So if $\nu_2(n)\to 0$ as $n\to\infty$, we must also have 
\[
\varepsilon_2(n) = \omega\left(\sqrt{\frac{\ln n}{ n}}\right).
\]
Second, consider the $(\varepsilon_1,\nu_1)$ pair. In $\nu_1$, the additional tail probability term converges to $0$ by the sub-gamma assumption, as proved in Lemma \ref{lem:sub-gamma_tail}, so $\nu_1$ converges to $0$. For $\varepsilon_1$, by definition, 
\[
\varepsilon_1 
= \sqrt{\frac{2\lambda\E[Z_{(n)}^2]}{B}\ln\left(\frac{1}{\nu_2}\right)}
= O\left(\sqrt{\frac{\ln n}{B}\ln\left(\frac{n\varepsilon_2^2}{\ln n}\right)}\right).
\]
Since $\varepsilon_2$ satisfies both $\varepsilon_2=o(1)$ and $\varepsilon_2= \omega(\sqrt{\ln n/n})$, letting $B=\Omega(\ln^2 n)$ suffices for $\varepsilon_1$ to converge to $0$.

Last, consider the $(\varepsilon_3,\nu_3)$ pair. Since $\varepsilon_3=\varepsilon_1$, hence $\varepsilon_3$ converges to $0$. In $\nu_3$, the tail probability term is proved to converge to $0$ above, and it remains to show $(p+(1-p)\nu_2^{\frac{1}{B}})^B$ converges to $0$. In (\ref{eq:g_limit}), we have seen in the conditional stability case that $g(p,\delta,B)\to 2\delta^{1-p}$ as $B\to\infty$, as long as $\ln\delta = o(B)$. Similarly, when $B(n)=\Omega(\ln^2 n)$, then $\ln\nu_2 = o(B)$, and we have 
\begin{align*}
    &\lim_{n\to\infty} B(n)\ln\left(p(n)+(1-p(n))\nu_2(n)^{\frac{1}{B(n)}}\right)\\
    = {} & \lim_{n\to\infty} B(n) \ln\left(p(n)+(1-p(n))\exp\left(\frac{\ln\nu_2(n)}{B(n)}\right)\right)\\
    = {} & \lim_{n\to\infty} B(n) \ln\left(p(n)+(1-p(n))\left(1+\frac{\ln\nu_2(n)}{B(n)}\right)\right)\\
    = {} & \lim_{n\to\infty} B(n) \ln\left(1+\frac{(1-p(n))\ln\nu_2(n)}{B(n)}\right)\\
    = {} & \lim_{n\to\infty} B(n) \ln\left(1+\frac{(1/e+O(1/n))\ln\nu_2(n)}{B(n)}\right)\\
    = {} & \lim_{n\to\infty} B(n) \ln\left(1+\frac{\ln\nu_2^{1/e}(n)}{B(n)}\right)\\
    = {} & \lim_{n\to\infty}  \ln\left[ \left(1+\frac{\ln\nu_2^{1/e}(n)}{B(n)}\right)^{B(n)} \right],
\end{align*}
which indicates
\begin{align*}
    \lim_{n\to\infty} (p(n)+(1-p(n))\nu_2(n)^{\frac{1}{B(n)}})^{B(n)} 
    = \lim_{n\to\infty}  \left(1+\frac{\ln\nu_2^{1/e}(n)}{B(n)}\right)^{B(n)} = \lim_{n\to\infty}  \nu_2^{1/e}(n) = 0,
\end{align*}
and the proof is completed.

\subsection{Discussion: beyond the sub-gamma assumption}
There is some subtlety in choosing the dependence of $\varepsilon_2$, $B$, and $\lambda$ on $n$. We consider the case that $\lambda$ is a fixed number above. However, if we make $\lambda$ also depend on $n$, then the assumption that $Y^2$ is sub-gamma can be removed. To see this, consider $\lim_{n\to\infty}\lambda(n)=\infty$. Then by Markov's inequality,
\[
\lim_{n\to\infty}\p(Z_{(n)}^2>\lambda(n)\E[Z_{(n)}^2])\leq\lim_{n\to\infty}\frac{1}{\lambda(n)}=0.
\]
So in this case, $\E[Z_{(n)}^2]$ can have a faster growth rate in $n$ and the sub-gamma assumption of $Y^2$ is not needed to control the tail probability. Consider 
$\E[Z_{(n)}^2]=\omega(\ln n)$, so that $Y^2$ has a heavier tail than sub-gamma. Requiring $\nu_2(n)$ to converge to $0$ then results in 
\[
\varepsilon_2(n)=\omega\left(\sqrt{\lambda(n)\E[Z_{(n)}^2]/n}\right),
\]
but $\varepsilon_2(n)$ also has to converge to 0, i.e., $\varepsilon_2(n)=o(1)$. This introduces another constraint on $\E[Z_{(n)}^2]$ that $\E[Z_{(n)}^2]=o(n/\lambda(n))$. A straightforward analysis shows that 
\[
B(n)=\Omega(\lambda(n)\E[Z_{(n)}^2]\ln (1/\nu_2(n)))
\]
suffices to guarantee the convergence of $\varepsilon_1$, $\varepsilon_3$, and $\nu_3$. 

Summing up, we find that by allowing $\lim_{n\to\infty}\lambda(n)=\infty$, one can generalize the results in Corollary \ref{cor:epsilon_nu_to_0} to the case that $Y^2$ follows a distribution beyond sub-gamma. The scaling of $\E[Z_{(n)}^2]$ is an indicator of the tail behavior of the underlying distribution of $Y^2$, and $\E[Z_{(n)}^2]$ that satisfies both
\[
\E[Z_{(n)}^2] = \omega(\ln n) \text{ and } \E[Z_{(n)}^2] = o(n/\lambda(n))
\]
represents a wide class of distributions with tails heavier than sub-gamma for proper dependence of $\lambda(n)$ on $n$. 

\begin{example}
   Consider some random variable $Y^2$ with $\E[Z_{(n)}^2]=\Theta(n^{1/4})$, and pick $\lambda(n)=\Theta(n^{1/4})$, $\varepsilon_2=\Theta(n^{-1/6})$, and $B=\Omega(n\ln n)$. Then consequently 
\begin{align*}
\nu_2 &{} = \frac{\lambda\E[Z_{(n)}^2]\eta}{n\varepsilon_2^2} = \Theta(n^{-1/6}),\\
\varepsilon_1=\varepsilon_3 &{} = \sqrt{\frac{2\lambda\E[Z_{(n)}^2]}{B}\ln\left(\frac{1}{\nu_2}\right)} = O(n^{-1/4}),\\
\nu_1 &{} = 2\nu_2+2\p( Z_{(n)}^2 > \lambda\E[Z_{(n)}^2]) = \Theta(n^{-1/6}) + O(n^{-1/4})=\Theta(n^{-1/6}),\\
\nu_3 &{} = 2(p+(1-p)\nu_2^{\frac{1}{B}})^B + 2\p( Z_{(n)}^2 > \lambda\E[Z_{(n)}^2]) = \Theta(n^{-1/(6e)})+O(n^{-1/4})=\Theta(n^{-1/(6e)}).
\end{align*} 
Hence $\varepsilon_{n,B}$ and $\nu_{n,B}$ converge to 0 for a heavy-tailed $Y^2$ beyond sub-gamma.
\end{example}

In Table \ref{table:stability_conditions}, we summarize the limiting stability parameters that can be achieved based on our theory, with properly chosen $\varepsilon_2$, $B$, and $\lambda$. This result shows the wide applicability of the RF stability. Sub-Gaussian is a subset of sub-gamma random variables with $c=0$, and the analysis is similar. For bounded $Y^2$, there is even no need to introduce the tail probability term because $|Y|\leq M$ implies $Z_{(n)}^2\leq M^2$, and $T'_1$ and $T'_3$ are naturally bounded as
\[
T'_1 \leq 2\exp\left(-\frac{Bt^2}{2M^2}\right),
\]
and
\[
T_3'\leq 2\left(p+(1-p)e^{-t^2/(2M^2)}\right)^B,
\]
respectively. Both of the upper bounds are deterministic, and the analysis is greatly simplified. All we need to do is replace $\lambda\E[Z_{(n)}^2]$ with $M^2$ and drop the tail probability term in Theorem \ref{thm:RF_stability}. 
\begin{corollary}
Assume training points in set $D$ and the test point $(X,Y)$ are iid, and $|Y|\leq M$. For the RF predictor $\RF$ consisting of $B$ trees and trained on $D$, we have 
    \begin{align}
    \p_{D,X,\bm{\xi,r}}
      \left(
      \left|\RF(X)-\RFi(X)\right|>\varepsilon_{n,B}
      \right) 
    \leq \nu_{n,B},
    \end{align}
where $\varepsilon_{n,B}=\sum_{i=1}^3\varepsilon_i$, and $\nu_{n,B}=\sum_{i=1}^3\nu_i$. The pair of $(\varepsilon_2,\nu_2)$ satisfies the derandomized RF stability condition 
\[
\p_{D,X}\left(
       \left| \rfall(X) - \rfi(X) \right| > \varepsilon_2
       \right) 
       \leq  
       \frac{M^2}{\varepsilon^2n}
       \left(\frac{p}{1-p}+\frac{q}{(1-p)^2}\right) = \nu_2,
\]
where $\varepsilon_1=\varepsilon_3=\sqrt{2M^2\ln(\frac{1}{\nu_2})/B}$, $\nu_1=2\nu_2$, and $\nu_3=g(p,\nu_2,B)$.
\end{corollary}

Based on our theory, for a random variable $Y^2$ with $\E[Z_{(n)}^2]=O(n)$, there is no hope to get both $\nu_{n,B}$ and $\varepsilon_{n,B}$ to converge to $0$. But as we show in Fig. \ref{fig:compare}, there are hints that the RF stability persists beyond our theory. This is because (\ref{eq:RF_range}) is used to provide the worst-case deviation bound between $\rfall$ and $\rfi$. If a more informative bound can be found to replace (\ref{eq:RF_range}), and the dependence of such a bound on $n$ is $o(\E[Z_{(n)}^2])$, then it is possible to find vanishing stability parameters even for strongly heavy-tailed random variables theoretically. Such an improved bound may also help boost the convergence rates of $\varepsilon_{n,B}$ and $\nu_{n,B}$. It might be  a future research direction to look for a better bound. 

\begin{table}
  \caption{Summary of random forest stability conditions}
  \label{table:stability_conditions}
  \centering
  \begin{tabular}{lll}
    \toprule
    \cmidrule(r){1-3}
    Scaling of $\E[Z_{(n)}^2]$     &  Value of $\lambda$     & Stability parameters \\
    \midrule
    $O(1)$ (bounded $Y^2$) & Unnecessary & $\varepsilon_{n,B}\to0$, $\nu_{n,B}\to0$    \\
    $O(\sqrt{\ln n})$ (sub-Gaussian $Y^2$) & $\text{Constant}$ & $\varepsilon_{n,B}\to0$, $\nu_{n,B}\to0$    \\
    $O(\ln n)$ (sub-gamma $Y^2$) & $\text{Constant}$  & $\varepsilon_{n,B}\to0$, $\nu_{n,B}\to0$    \\
    Between $\omega(\ln n)$ and $o(n/\lambda(n))$ & $\lambda\to\infty$ & $\varepsilon_{n,B}\to0$, $\nu_{n,B}\to0$ \\
    \bottomrule
  \end{tabular}
\end{table}

\subsection{More examples of RF stability}
We consider the RF stability on four real datasets, publicly available at UCI Machine Learning Repository \cite{asuncion2007uci}. We name them as \texttt{Concrete} \cite{yeh1998modeling}, \texttt{Airfoil} \cite{brooks1989airfoil}, \texttt{Bioconcentration} \cite{grisoni2016investigating}, and \texttt{ Naval }\cite{coraddu2016machine}. For each dataset, we investigate three aspects:

\paragraph{Marginal distribution of $Y$.} As shown in Fig. \ref{fig:more_compare} (left column), all the density plots of $Y$ seem to have a light tail. In fact, the response in many real datasets is bounded or narrowly distributed within a given interval, which has to do with the physical constraints. For example, the strength of some material is determined by the underlying chemical bond strength, and cannot be arbitrarily large.

\paragraph{Difference between the RF predictor and OOB predictors.} To this end, we randomly split every dataset into two parts with equal size $n$. One is for training and the other for testing. Based on $n$ training points, we have an RF predictor $\RF$, as well as $n$ OOB predictors $\RFi$. We fix $B=1000$ in all cases. For each of these predictors, we compute its error on $n$ test points, and in Fig. \ref{fig:more_compare} (middle column) we plot the density of the absolute difference $|\RF(X)-\RFi(X)|$ for all $n$ OOB predictors. Also, we calculate the 0.95 quantile of each empirical distribution of the difference, say, $\hat\varepsilon_{0.95,i}, i\in[n]$. We let $\hat\varepsilon_{n,B} = \max_i \hat\varepsilon_{0.95,i}$ as an estimate of $\varepsilon_{n,B}$ given $\nu_{n,B}=0.05$. 

\paragraph{Comparison between prediction error and OOB error.} For the RF predictor, we calculate its prediction error on $n$ test points to come up with a density plot of such prediction errors. For each OOB predictor, we calculate its corresponding OOB error, and we have a density plot based on $n$ such OOB errors. In Fig. \ref{fig:more_compare} (right column), we plot both of prediction and OOB error. The similarity between them, especially on the right tail, provides much credence to the idea of constructing PIs using the OOB error. 

We train the RFs using default parameters except that the number of trees is fixed to be $B=1000$ in all cases. The training can be done within a few minutes on a laptop. 


\begin{figure}
  \centering
  \begin{subfigure}{0.3\textwidth}
  \includegraphics[width=\linewidth]{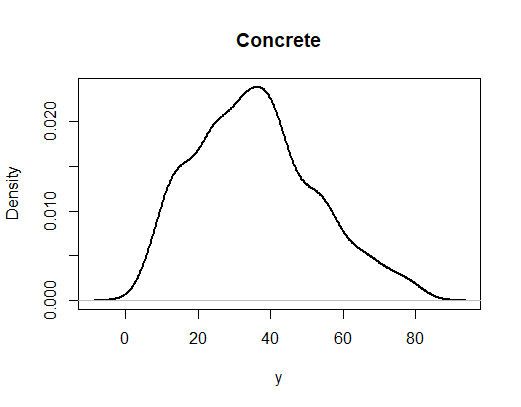}
  \end{subfigure}%
  \begin{subfigure}{0.3\textwidth}
  \includegraphics[width=\linewidth]{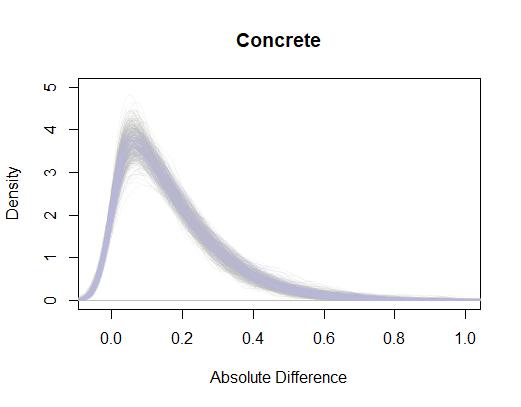}
  \end{subfigure}%
  \begin{subfigure}{0.3\textwidth}
  \includegraphics[width=\linewidth]{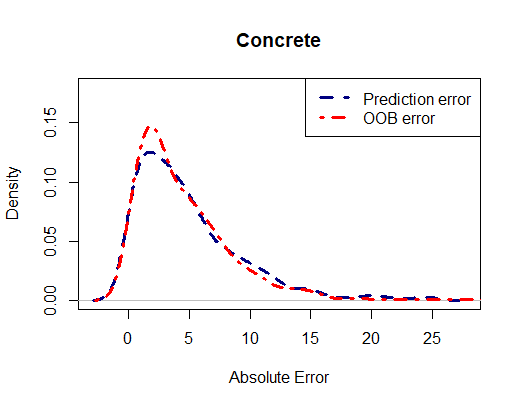}    
  \end{subfigure}\par\medskip
  \begin{subfigure}{0.3\textwidth}
  \includegraphics[width=\linewidth]{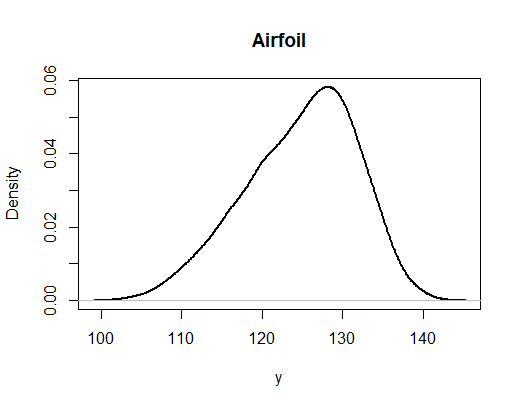}
  \end{subfigure}%
  \begin{subfigure}{0.3\textwidth}
  \includegraphics[width=\linewidth]{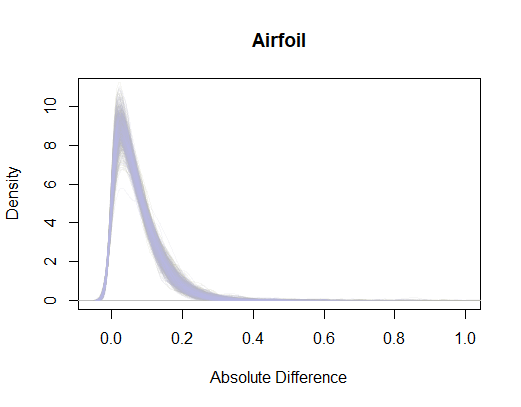}
  \end{subfigure}%
  \begin{subfigure}{0.3\textwidth}
  \includegraphics[width=\linewidth]{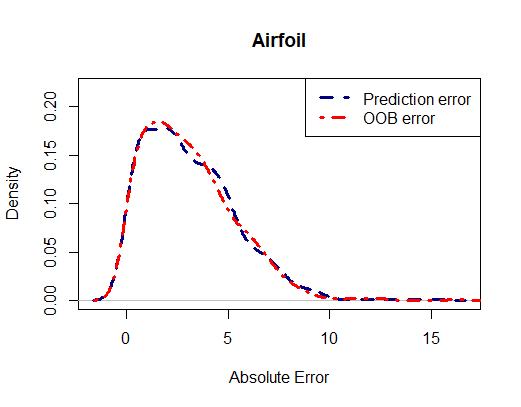}    
  \end{subfigure}\par\medskip
  \begin{subfigure}{0.3\textwidth}
  \includegraphics[width=\linewidth]{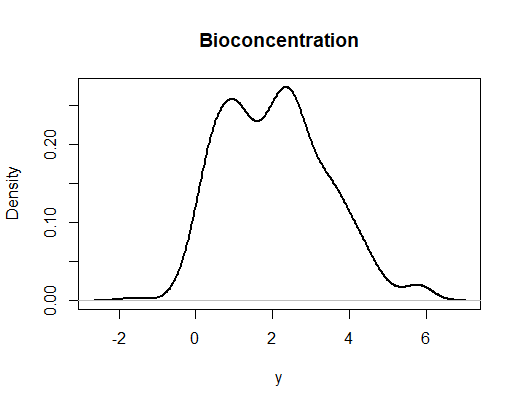}
  \end{subfigure}%
  \begin{subfigure}{0.3\textwidth}
  \includegraphics[width=\linewidth]{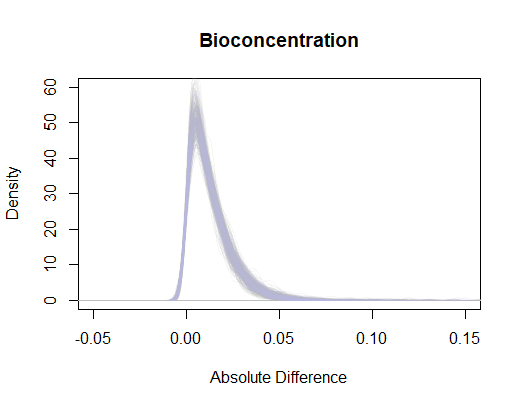}
  \end{subfigure}%
  \begin{subfigure}{0.3\textwidth}
  \includegraphics[width=\linewidth]{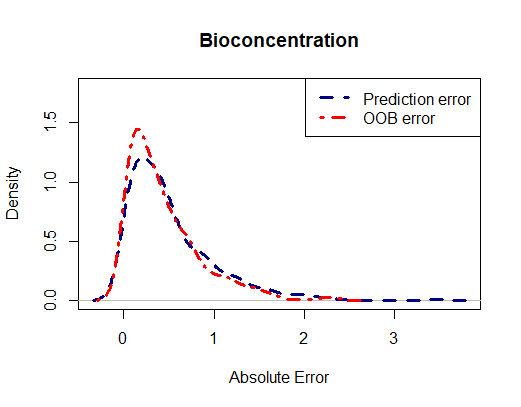}   
  \end{subfigure}\par\medskip
  \begin{subfigure}{0.3\textwidth}
  \includegraphics[width=\linewidth]{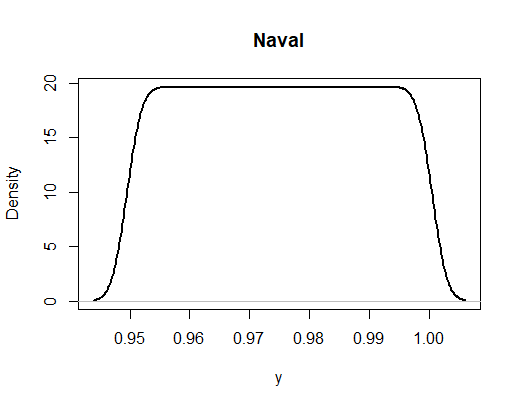}
  \end{subfigure}%
  \begin{subfigure}{0.3\textwidth}
  \includegraphics[width=\linewidth]{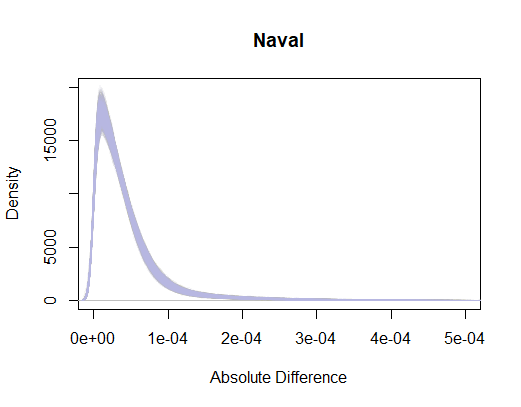}
  \end{subfigure}%
  \begin{subfigure}{0.3\textwidth}
  \includegraphics[width=\linewidth]{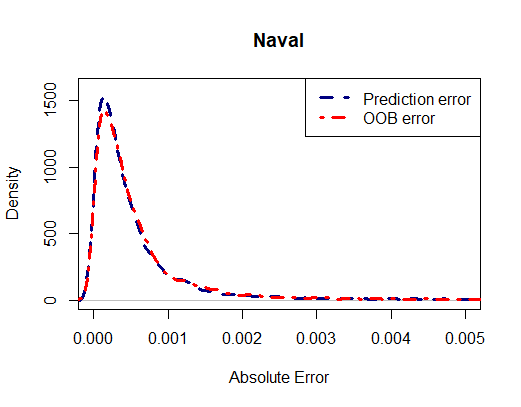}    
  \end{subfigure}
  \caption{Left column: Density plots of $Y$. Middle column: Density plots of $|\RF(X)-\RFi(X)|$. Right column: Density plots of $|Y-\RF(X)|$ and $|Y_i-\RFi(X_i)|$. Row 1: \texttt{Concrete} dataset. Numerically, we find $\hat \varepsilon_{n,B}\approx0.62, \hat \nu_{n,B}=0.05$ for $n=515$. Row 2: \texttt{Airfoil} dataset with $\hat \varepsilon_{n,B}\approx0.25, \hat \nu_{n,B}=0.05$ for $n=751$. Row 3: \texttt{Bioconcentration} dataset with $\hat \varepsilon_{n,B}=0.05, \hat \nu_{n,B}\approx0.05$ for $n=389$. Row 4: \texttt{Naval} dataset with $\hat \varepsilon_{n,B}\approx0.00016, \hat \nu_{n,B}=0.05$ for $n=5967$. }
  \label{fig:more_compare}
\end{figure}


\section{Proof of Theorem \ref{thm:PI_lower_bound}}
Recall that each of $\varepsilon_{n,B}$ and $\nu_{n,B}$ can be written as the sum of three terms: \[\varepsilon_{n,B}=\varepsilon_1+\varepsilon_2 + \varepsilon_3, \text{ }\nu_{n,B}=\nu_1+\nu_2 + \nu_3,\]
and we have established in the proof of Theorem \ref{thm:RF_stability} that
\begin{align*}
    &{} \p\left(|\rfall(X)-\RF(X)|>\varepsilon_1\right) \leq \nu_1,\\
    &{} \p\left(|\rfall(X)-\rfi(X)|>\varepsilon_2\right) \leq \nu_2,\\
    &{} \p\left(|\RFi(X)-\rfi(X)|>\varepsilon_3\right) \leq \nu_3.
\end{align*}
Now \emph{assume} the test point $(X,Y)\equiv(X_{n+1},Y_{n+1})$ is accessible to us, and define $n+1$ derandomized RF predictors as follows:
\[
\widetilde{\rfall}^{\backslash i}\equiv\widetilde{\rfall}^{\backslash i}((X_1,Y_1),\ldots,(X_{i-1},Y_{i-1}),(X_{i+1},Y_{i+1}),\ldots,(X_{n+1},Y_{n+1})),
\]
meaning that $\widetilde{\rfall}^{\backslash i}$ is trained on $n$ pairs of data points without $(X_i,Y_i)$. Furthermore, define for $i\in[n+1]$ that
\[
\widetilde{r}_i = |Y_i-\widetilde{\rfall}^{\backslash i}(X_i)|.
\]

\begin{lemma}\label{lem:tilde_r_n+1}
    For $\alpha\in(1/(n+1),1)$, let $\widetilde r_i, i\in[n+1]$ be defined above. Then
    \[
    \p\left(\widetilde r_{n+1}\leq q_{n,\alpha}\{\widetilde{r}_i\}\right)\geq 1-\alpha,
    \]
where $q_{n,\alpha}\{\widetilde{r}_i\}$ denotes the $\lceil (1-\alpha)(n+1) \rceil$-th smallest value of $\{\widetilde{r}_1,\ldots,\widetilde{r}_n\}.$ If we further assume there are no ties among $\widetilde r_i$, then we also have
    \[
    \p\left(\widetilde r_{n+1}\leq q_{n,\alpha}\{\widetilde{r}_i\}\right)\leq 1-\alpha+ \frac{1}{n+1}.
    \]
\end{lemma}
\begin{proof}
    This proof is provided for completeness. For the first part, note that all $\widetilde r_i$ are exchangeable. Hence the rank of each $\widetilde r_i$ is uniformly distributed on $[n+1]$, and 
    \begin{align*}
        \p\left(\widetilde r_{n+1} > q_{n,\alpha}\{\widetilde r_i\}\right) 
        &{} \leq \frac{n+1 - \lceil (1-\alpha)(n+1) \rceil}{n+1}\\
        &{} = \frac{\lfloor \alpha(n+1) \rfloor}{n+1} && (\text{note } n+1 = \lceil(n+1)(1-\alpha)\rceil + \lfloor(n+1)\alpha\rfloor)\\
        &{} \leq \frac{\alpha(n+1)}{n+1}\\
        &{} =\alpha.
    \end{align*}
    For the second part, if there are no ties among all $\{\widetilde r_i\}$, then we have 
    \[
      \p(\widetilde r_{n+1}<\widetilde r_{(1)}) = \p\left(\widetilde r_{n+1} \text{ is the smallest among }\{\widetilde r_i\}_{i=1}^{n+1}\right) = \frac{1}{n+1},    
    \]
    and for each $i\in[n-1]$, we have
    \begin{align*}
        \p\left(\widetilde r_{(i)}<\widetilde r_{n+1}<\widetilde r_{(i+1)}\right) = \p\left(\widetilde r_{n+1} \text{ is the $(i+1)$-th smallest among } \{\widetilde r_i\}_{i=1}^{n+1}\right) = \frac{1}{n+1}.
    \end{align*}
    As a result,
    \begin{align*}
        \p\left(\widetilde r_{n+1}\leq q_{n,\alpha}\{\widetilde r_i\}\right) 
        & {} = \p\left(\widetilde r_{n+1}\leq \widetilde r_{(\lceil(1-\alpha)(n+1)\rceil)}\right)\\
        &{} = \p\left(\widetilde r_{n+1}< \widetilde r_{(\lceil(1-\alpha)(n+1)\rceil)}\right)\\
        &{} = \p\left(\widetilde r_{n+1}< \widetilde r_{(1)}\right) + \sum_{i=1}^{\lceil(1-\alpha)(n+1)\rceil-1}\p\left(\widetilde r_{(i)}<\widetilde r_{n+1}<\widetilde r_{(i+1)}\right)\\
        &{} = \frac{\lceil(1-\alpha)(n+1)\rceil}{n+1}.
    \end{align*}
    By the definition of the ceiling function, we have 
    \[
    (1-\alpha)(n+1)\leq \lceil(1-\alpha)(n+1)\rceil\leq (1-\alpha)(n+1) + 1.
    \]
    Therefore, we conclude that if there are no ties between $\{\widetilde r_i\}$, then
    \[
    1-\alpha \leq \p\left(\widetilde r_{n+1}\leq q_{n,\alpha}\{\widetilde r_i\}\right) \leq 1-\alpha + \frac{1}{n+1},
    \]
    which completes the proof.
\end{proof}

By the first part of the above lemma and the definition of $\widetilde r_{n+1}$, we know that
\[
\p\left(
|Y_{n+1}-\widetilde{\rfall}^{\backslash (n+1)}(X_{n+1})|
\leq q_{n,\alpha}\{\widetilde{r}_i\}
\right)\geq 1-\alpha.
\]
But $\widetilde{\rfall}^{\backslash (n+1)}$ is just $\rfall$, so we have
\[
\p\left(
|Y_{n+1}-\rfall(X_{n+1})|
\leq q_{n,\alpha}\{\widetilde{r}_i\}
\right)\geq 1-\alpha.
\]
Of course, $\widetilde{r}_i$ are unknown, and we eventually will need to replace them with $R_i$, but before that, we use $r_i$ instead. Before proceeding, a useful lemma that connects $\{\widetilde{r}_i\}$, $\{r_i\}$, and $\{R_i\}$ is given below.

\begin{lemma}\label{lem:compare}
    Suppose there are $n$ pairs of real numbers $(a_i,b_i), i\in\{1,\ldots,n\}$. Let $a_{(1)}\leq\ldots\leq a_{(n)}$ and $b_{(1)}\leq\ldots\leq b_{(n)}$. For any $\varepsilon\in\R$, if  $b_{(k)} >a_{(j)}+\varepsilon$ for some $j$ and $k$, then there are at least $j-k+1$ pairs of $(a_i,b_i)$ such that $b_i > a_i+\varepsilon$. In particular, if $k = \lceil (n+1)(1-\alpha) \rceil$ and $j=\lceil (n+1)(1-\alpha') \rceil$ with $\alpha' = \alpha - \delta_\alpha$ and $\delta_\alpha\geq0$, then $j-k+1\geq (n+1)\delta_\alpha$. 
\end{lemma}
\begin{remark}
    It might be more convenient to state the lemma as that if $q_{n,\alpha}\{b_i\}>q_{n,\alpha'}\{a_i\}+\varepsilon$ with $\alpha'=\alpha-\delta_\alpha$, then there are at least $(n+1)\delta_\alpha$ pairs of $(a_i,b_i)$ such that $b_i>a_i+\varepsilon$.
\end{remark}
\begin{proof}
    This result was directly obtained in \cite{barber2021predictive} ``by definition of quantiles,''  while a more complete proof might be helpful. Without loss of generality, we take $\varepsilon=0$. (Otherwise, we consider $\{a_i',b_i\}$ instead, where $a_i'=a_i+\varepsilon$.) We first define three sets as follows.
    \begin{align*}
    & S_A\equiv \{i:a_i\leq a_{(j)}\},\\ 
    & S_B\equiv\{i:b_i\geq b_{(k)}\},\\
    & S_C\equiv\{i:a_i<b_i\}.
    \end{align*}
It is obvious that for any $i\in S_A\cap S_B$, we have $a_i\leq a_{(j)} < b_{(k)} \leq b_i$. Hence $i\in S_C$. That is, 
\[
S_A \cap S_B \subseteq S_C.
\]
As a consequence, $|S_A\cap S_B|\leq |S_C|$, where $|\cdot|$ denotes the cardinality of a set. Then, note that 
\[
S_A=(S_A\cap S_B)\cup(S_A\cap S_B^c),
\]
where $S_B^c$ is the complement of $S_B$, and 
\[
|S_A| = |S_A \cap S_B| + |S_A \cap S_B^c|.
\]
Therefore, we have
\[
|S_C| \geq  |S_A \cap S_B| = |S_A| - |S_A \cap S_B^c| \geq |S_A| - |S_B^c| = j - (k-1) = j - k + 1.
\]
Now consider $j=\lceil (n+1)(1-\alpha') \rceil$ and $k = \lceil (n+1)(1-\alpha) \rceil$ with $\alpha' = \alpha - \delta_\alpha$. Note for any $\beta\in(0,1)$ that $n+1 = \lceil(n+1)(1-\beta)\rceil + \lfloor(n+1)\beta\rfloor$. Then we have
\begin{align*}
    j-k+1
    & {} = \lceil (n+1)(1-\alpha') \rceil - \lceil (n+1)(1-\alpha) \rceil + 1\\
    & {} = \lfloor(n+1)\alpha\rfloor - \lfloor(n+1)\alpha'\rfloor + 1\\
    & {} = \lfloor(n+1)(\alpha'+\delta_\alpha)\rfloor - \lfloor(n+1)\alpha'\rfloor + 1\\
    & {} \geq \lfloor(n+1)\alpha'\rfloor+ \lfloor(n+1)\delta_\alpha\rfloor - \lfloor(n+1)\alpha'\rfloor + 1\\
    & {} = \lfloor(n+1)\delta_\alpha\rfloor + 1\\
    & {} \geq (n+1)\delta_\alpha,
\end{align*}
completing the proof.
\end{proof}
We are now ready to prove Theorem \ref{thm:PI_lower_bound}. Basically, we use the stability property once and the concentration of measure twice to establish probabilistic deviation bounds for $|\rfall(X)-\rfi(X)|$, $|\rfall(X)-\RF(X)|$, and $|\rfi(X)-\RFi(X)|$, respectively.
\subsection{Using the stability property to control $|\rfall(X)-\rfi(X)|$}
Following the same idea as in \cite{barber2021predictive}, we consider the event
\[
A_l = ``\text{$q_{n,\alpha}\{\widetilde r_i\} > q_{n,\alpha_2}\{r_i\}+\varepsilon_2$}", 
\]
where $\alpha_2=\alpha-\sqrt{\nu_2}$. We then have that
\begin{align*}
    1-\alpha 
    &{} \leq \p\left(
|Y_{n+1}-\rfall(X_{n+1})|
\leq q_{n,\alpha}\{\widetilde{r}_i\}
\text{ and } A_l \right) + \p\left(
|Y_{n+1}-\rfall(X_{n+1})|
\leq q_{n,\alpha}\{\widetilde{r}_i\}
\text{ and } A_l^c \right)\notag\\
&{} \leq \p(A_l) + \p\left(
\widetilde r_{n+1}
\leq q_{n,\alpha_2}\{r_i\} + \varepsilon_2
\right),
\end{align*}
where $A_l^c$ is the complement of $A_l$.
Now, we calculate $\p(A_l)$. By Lemma \ref{lem:compare}, event $A_l$ implies that there are at least $(n+1)\sqrt{\nu_2}$ pairs of $(\widetilde r_i,r_i)$ such that $\widetilde r_i > r_i+\varepsilon_2$, and thus
\begin{align*}
    \p(A_l) 
    & {} \leq \p\left(\sum_{i=1}^n\mathbb{I}\{\widetilde r_i>r_i+\varepsilon_2\}\geq (n+1)\sqrt{\nu_2}\right)\\
    &{} \leq \frac{\E\left[\sum_{i=1}^n\mathbb{I}\{\widetilde r_i>r_i+\varepsilon_2\}\right]}{(n+1)\sqrt{\nu_2}} \text{ (by Markov's inequality)}\\
    & {} =\frac{n\E\left[\mathbb{I}\{\widetilde r_i>r_i+\varepsilon_2\}\right]}{(n+1)\sqrt{\nu_2}} \text{ (by iid data)}\\
    &{} = \frac{n\p\left(\widetilde r_i>r_i+\varepsilon_2\right)}{(n+1)\sqrt{\nu_2}}\\
    &{} = \frac{n\p\left(|Y_i-\widetilde{\rfall}^{\backslash i}(X_i)|>|Y_i-\rfi(X_i)|+\varepsilon_2\right)}{(n+1)\sqrt{\nu_2}}\\
    &{} = \frac{n\p\left(|Y_i-\widetilde{\rfall}^{\backslash i}(X_i)|>|Y_i-\widetilde{\rfall}^{\backslash(n+1,i)}(X_i)|+\varepsilon_2\right)}{(n+1)\sqrt{\nu_2}} \text{ (by definitions of $\widetilde{\rfall}^{\backslash i},\rfall,\rfi$)}\\
    &{} = \frac{n\p\left(|Y_{n+1}-\widetilde{\rfall}^{\backslash (n+1)}(X_{n+1})|>|Y_{n+1}-\widetilde{\rfall}^{\backslash(i,n+1)}(X_{n+1})|+\varepsilon_2\right)}{(n+1)\sqrt{\nu_2}} \text{ (by iid data)}\\
    &{} = \frac{n\p\left(|Y_{n+1}-\rfall(X_{n+1})|>|Y_{n+1}-\rfi(X_{n+1})|+\varepsilon_2\right)}{(n+1)\sqrt{\nu_2}} \text{ (by definitions of $\widetilde{\rfall}^{\backslash i},\rfall,\rfi$)}\\
    &{} \leq \frac{n\p\left(|Y_{n+1}-\rfall(X_{n+1})-Y_{n+1}+\rfi(X_{n+1})|>\varepsilon_2\right)} {(n+1)\sqrt{\nu_2}} \text{ (because $|a|-|b|\leq |a-b|$)}\\
    &{} =\frac{n\p\left(|\rfall(X_{n+1})-\rfi(X_{n+1})|>\varepsilon_2\right)} {(n+1)\sqrt{\nu_2}} \\
    &{} \leq \frac{n\nu_2}{(n+1)\sqrt{\nu_2}} \text{ (by stability of $\rfall$)}\\
    &{} \leq \sqrt{\nu_2}.
\end{align*}
As a consequence, we have
\[
\p\left(
|Y_{n+1}-\rfall(X_{n+1})|
\leq q_{n,\alpha_2}\{r_i\} + \varepsilon_2
\right) \geq 1-\alpha-\p(A_l) \geq 1-\alpha-\sqrt{\nu_2} = 1-\alpha_2-2\sqrt{\nu_2}.
\]
That is to say, we have a reduced lower bound of coverage using an $\varepsilon_2$-inflated interval constructed from $\{r_i\}$. However, $\{r_i\}$ is unknown, and we want to further approximate $r_i$ by $R_i$.

\subsection{Using the concentration of measure to control $|\rfi(X)-\RFi(X)|$}

To this end, we similarly define another event:
\[
A_l' = ``\text{$q_{n,\alpha_2}\{r_i\} > q_{n,\alpha_3}\{R_i\}+\varepsilon_3$}",
\]
where $\alpha_3=\alpha_2-\sqrt{\nu_3}$, and we have
\begin{align*}
 1-&{}\alpha_2-2\sqrt{\nu_2}\\ 
 &{}\leq \p\left(
|Y_{n+1}-\rfall(X_{n+1})|
\leq q_{n,\alpha_2}\{r_i\} + \varepsilon_2 \text{ and } A_l'
\right) + \p\left(
|Y_{n+1}-\rfall(X_{n+1})|
\leq q_{n,\alpha_2}\{r_i\} + \varepsilon_2
\text{ and } A^{'c}_l\right)\\
&{} \leq \p(A_l') + \p\left(
|Y_{n+1}-\rfall(X_{n+1})|
\leq q_{n,\alpha_3}\{R_i\} + \varepsilon_2+\varepsilon_3\right),
\end{align*}
where $A^{'c}_l$ is the complement of $A'_l$. Applying Lemma \ref{lem:compare} again yields
\begin{align*}
    \p(A'_l)
    &{} \leq \p\left(\sum_{i=1}^n\mathbb{I}\{r_i>R_i+\varepsilon_3\}\geq (n+1)\sqrt{\nu_3}\right)\\
    &{} \leq \frac{\E\left[\sum_{i=1}^n\mathbb{I}\{r_i>R_i+\varepsilon_3\}\right]}{(n+1)\sqrt{\nu_3}} \text{ (by Markov's inequality)}\\
    &{} \leq \frac{n\p\left(r_i>R_i+\varepsilon_3\right)}{(n+1)\sqrt{\nu_3}}\text{ (by iid data)}\\
    &{} = \frac{n\p\left(|Y_i-\rfi(X_i)|>|Y_i-\RFi(X_i)|+\varepsilon_3\right)}{(n+1)\sqrt{\nu_3}}\\
    &{} \leq \frac{n\p\left(|Y_i-\rfi(X_i)-Y_i+\RFi(X_i)|>\varepsilon_3\right)}{(n+1)\sqrt{\nu_3}} \text{ (because $|a|-|b|\leq |a-b|$)}\\
    &{} = \frac{n\p\left(|\rfi(X_i)-\RFi(X_i)|>\varepsilon_3\right)}{(n+1)\sqrt{\nu_3}}\\
    &{} = \frac{n\p\left(|\rfi(X)-\RFi(X)|>\varepsilon_3\right)}{(n+1)\sqrt{\nu_3}} \text{ (because $X_i$ and $X$ are iid, and $\rfi$ and $\RFi$ do not depend on $X_i$)}\\
    &{} \leq \frac{n\nu_3}{(n+1)\sqrt{\nu_3}} \text{ (by concentration of measure)}\\
    &{} \leq \sqrt{\nu_3}.
\end{align*}
So we have
\begin{align*}
\p\left(
|Y_{n+1}-\rfall(X_{n+1})|
\leq q_{n,\alpha_3}\{R_i\} + \varepsilon_2+\varepsilon_3\right) 
&{} \geq 1-\alpha_2-2\sqrt{\nu_2}-\p(A'_l)\geq 1-\alpha_3 -2\sqrt{\nu_2} - 2\sqrt{\nu_3}.
\end{align*}

\subsection{Using the concentration of measure to control $|\rfall(X)-\RF(X)|$}
Finally, we need to replace $\rfall(X_{n+1})$ by $\RF(X_{n+1})$. Note that for $t>0$,
\begin{align*}
\p\left(
|Y-\RF(X)|> t+\varepsilon_1
\right)
&{} = \p\left(
|Y-\rfall(X)+\rfall(X)-\RF(X)|> t+\varepsilon_1
\right)\\
&{} \leq \p\left(
|Y-\rfall(X)|> t
\right) + \p\left(
|\rfall(X)-\RF(X)|> \varepsilon_1
\right)\\
&{} \leq \p\left(
|Y-\rfall(X)|> t
\right) + \nu_1,
\end{align*}
which implies
\[
\p\left(
|Y-\RF(X)|\leq t+\varepsilon_1
\right) \geq \p\left(
|Y-\rfall(X)|\leq t
\right) - \nu_1.
\]
Let $t=q_{n,\alpha_3}\{R_i\} + \varepsilon_2+\varepsilon_3$, and we arrive at
\begin{align*}
\p\left(
|Y-\RF(X)|\leq q_{n,\alpha_3}\{R_i\} + \varepsilon_2+\varepsilon_3 +\varepsilon_1
\right) 
&{} \geq \p\left(
|Y-\rfall(X)|\leq q_{n,\alpha_3}\{R_i\} + \varepsilon_2+\varepsilon_3
\right) - \nu_1\\
&{} \geq 1-\alpha_3 -2\sqrt{\nu_2} - 2\sqrt{\nu_3} -\nu_1,
\end{align*}
which completes the proof of Theorem \ref{thm:PI_lower_bound}. Moreover, for $\nu_1\in(0,1)$, we have $\sqrt{\nu_1}>\nu_1$. So
\[
2\sqrt{\nu_2} + 2\sqrt{\nu_3} + \nu_1 \leq 2(\sqrt{\nu_2}+\sqrt{\nu_3}+\sqrt{\nu_1}) = 6 \times \frac{\sqrt{\nu_2}+\sqrt{\nu_3}+\sqrt{\nu_1}}{3}\leq 6\times\sqrt{(\nu_1+\nu_2+\nu_3)/3}=2\sqrt{3}\sqrt{\nu_{n,B}}.
\]
Hence we eventually have
\[
\p\left(
|Y-\RF(X)|\leq q_{n,\alpha}\{R_i\} + \varepsilon_{n,B}
\right) \leq 1-\alpha - 2\sqrt{3}\sqrt{\nu_{n,B}},
\]
which leads to the informal version of the theorem.

\section{Proof of Theorem \ref{thm:PI_upper_bound}}

\subsection{Using the stability property to control $|\rfall(X)-\rfi(X)|$}
Since we have assumed $\{\widetilde r_i\},i\in[n+1]$ have no ties, by the second part of Lemma \ref{lem:tilde_r_n+1}, we have for the test point $(X,Y)=(X_{n+1},Y_{n+1})$ that
\[
\p\left(|Y_{n+1}-\widetilde \rfall^{\backslash (n+1)}(X_{n+1})|\leq q_{n,\alpha}\{\widetilde r_i\}\right) 
= \p\left(\widetilde r_{n+1}\leq q_{n,\alpha}\{\widetilde r_i\}\right) 
\leq 1-\alpha + \frac{1}{n+1}.
\]
Now consider the event that
\[
A_u = ``q_{n,\alpha_2}\{r_i\}-\varepsilon_2 \leq q_{n,\alpha}\{\widetilde r_i\}",
\]
where $\alpha_2=\alpha+\sqrt{\nu_2}$. We denote $A_u^c$ as the complement of $A_u$. Note that in this case the ancillary quantity $\alpha_2$ is greater than $\alpha$, while in the proof of Theorem \ref{thm:PI_lower_bound}, $\alpha_2$ is less than $\alpha$. We then have
\begin{align*}
    \p\left(\widetilde r_{n+1}\leq q_{n,\alpha_2}\{r_i\} - \varepsilon_2\right)
    &{} = \p\left(\widetilde r_{n+1}\leq q_{n,\alpha_2}\{r_i\} - \varepsilon_2 \text{ and } A_u\right) + \p\left(\widetilde r_{n+1}\leq q_{n,\alpha_2}\{r_i\} - \varepsilon_2 \text{ and } A^c_u\right)\\
    &{} \leq \p\left(\widetilde r_{n+1}\leq q_{n,\alpha_2}\{r_i\} - \varepsilon_2 \text{ and } A_u\right) + \p(A^c_u)\\
    &{} \leq \p\left(\widetilde r_{n+1}\leq q_{n,\alpha}\{\widetilde r_i\}\right) + \p(A^c_u)\\
    &{} \leq 1-\alpha+\frac{1}{n+1}+\p(A^c_u).
\end{align*}

Next, we bound $\p(A^c_u)$. By Lemma \ref{lem:compare}, event $A_u^c$ implies there are at least $(n+1)\sqrt{\nu_2}$ pairs of $(r_i,\widetilde r_i)$ such that $ r_i > \widetilde r_i+\varepsilon_2$, and thus
\begin{align*}
    \p(A_u^c) 
    & {} \leq \p\left(\sum_{i=1}^n\mathbb{I}\{ r_i>\widetilde r_i+\varepsilon_2\}\geq (n+1)\sqrt{\nu_2}\right)\\
    &{} \leq \frac{\E\left[\sum_{i=1}^n\mathbb{I}\{ r_i>\widetilde r_i+\varepsilon_2\}\right]}{(n+1)\sqrt{\nu_2}} \text{ (by Markov's inequality)}\\
    & {} =\frac{n\E\left[\mathbb{I}\{ r_i>\widetilde r_i+\varepsilon_2\}\right]}{(n+1)\sqrt{\nu_2}} \text{ (by iid data)}\\
    &{} = \frac{n\p\left(r_i>\widetilde r_i+\varepsilon_2\right)}{(n+1)\sqrt{\nu_2}}\\
    &{} = \frac{n\p\left(|Y_i-\rfi(X_i)|>|Y_i-\widetilde \rfall^{\backslash i}(X_i)|+\varepsilon_2\right)}{(n+1)\sqrt{\nu_2}}\\
    &{} = \frac{n\p\left(|Y_i-\widetilde{\rfall}^{\backslash (n+1,i)}(X_i)|>|Y_i-\widetilde{\rfall}^{\backslash i}(X_i)|+\varepsilon_2\right)}{(n+1)\sqrt{\nu_2}} \text{ (by definitions of $\widetilde{\rfall}^{\backslash i},\rfall,\rfi$)}\\
    &{} = \frac{n\p\left(|Y_{n+1}-\widetilde{\rfall}^{\backslash (i,n+1)}(X_{n+1})|>|Y_{n+1}-\widetilde{\rfall}^{\backslash(n+1)}(X_{n+1})|+\varepsilon_2\right)}{(n+1)\sqrt{\nu_2}} \text{ (by iid data)}\\
    &{} = \frac{n\p\left(|Y_{n+1}-\rfi(X_{n+1})|>|Y_{n+1}-\rfall(X_{n+1})|+\varepsilon_2\right)}{(n+1)\sqrt{\nu_2}} \text{ (by definitions of $\widetilde{\rfall}^{\backslash i},\rfall,\rfi$)}\\
    &{} \leq \frac{n\p\left(|Y_{n+1}-\rfi(X_{n+1})-Y_{n+1}+\rfall(X_{n+1})|>\varepsilon_2\right)} {(n+1)\sqrt{\nu_2}} \text{ (because $|a|-|b|\leq |a-b|$)}\\
    &{} =\frac{n\p\left(|\rfall(X_{n+1})-\rfi(X_{n+1})|>\varepsilon_2\right)} {(n+1)\sqrt{\nu_2}} \\
    &{} \leq \frac{n\nu_2}{(n+1)\sqrt{\nu_2}} \text{ (by stability of $\rfall$)}\\
    &{} \leq \sqrt{\nu_2}.
\end{align*}
Hence we have
\begin{align*}
    \p\left(\widetilde r_{n+1}\leq q_{n,\alpha_2}\{r_i\} - \varepsilon_2\right) \leq 1-\alpha+\frac{1}{n+1}+\sqrt{\nu_2}=1-\alpha_2+\frac{1}{n+1}+2\sqrt{\nu_2}.
\end{align*}
Note that this proof also works for a general stable algorithm, as stated in Corollary \ref{cor:general_stable_upper_bound}.

\subsection{Using the concentration of measure to control $|\rfi(X)-\RFi(X)|$}
We further define an event
\[
A'_u = ``q_{n,\alpha_3}\{R_i\}-\varepsilon_3\leq q_{n,\alpha_2}\{r_i\}",
\]
where $\alpha_3=\alpha_2+\sqrt{\nu_3}$. We denote $A^{'c}_u$ as the complement of $A'_u$. Again, we increase rather than decrease $\alpha_3$ in this case, as opposed to in the proof of the lower bound. 
We then have
\begin{align*}
    \p(\widetilde r_{n+1} &{} \leq q_{n,\alpha_3}\{R_i\} - \varepsilon_2-\varepsilon_3)\\
    &{} = \p\left(\widetilde r_{n+1}\leq q_{n,\alpha_3}\{R_i\} - \varepsilon_2 -\varepsilon_3 \text{ and } A'_u\right) + \p\left(\widetilde r_{n+1}\leq q_{n,\alpha_3}\{R_i\} - \varepsilon_2 -\varepsilon_3 \text{ and } A^{'c}_u\right)\\
    &{} \leq \p\left(\widetilde r_{n+1}\leq q_{n,\alpha_3}\{R_i\} - \varepsilon_2 -\varepsilon_3 \text{ and } A'_u\right) + \p(A^{'c}_u)\\
    &{} \leq \p\left(\widetilde r_{n+1}\leq q_{n,\alpha_2}\{r_i\}-\varepsilon_2\right) + \p(A^{'c}_u)\\
    &{} \leq 1-\alpha_2+\frac{1}{n+1}+2\sqrt{\nu_2}+\p(A^{'c}_u). 
\end{align*}
By Lemma \ref{lem:compare}, $A^{'c}_u$ implies there exist at least $(n+1)\sqrt{\nu_3}$ pairs of $(R_i,r_i)$ such that $R_i>r_i+\varepsilon_3$, and thus 
\begin{align*}
    \p(A^{'c}_u)
    &{} \leq \p\left(\sum_{i=1}^n\mathbb{I}\{R_i>r_i+\varepsilon_3\}\geq (n+1)\sqrt{\nu_3}\right)\\
    &{} \leq \frac{\E\left[\sum_{i=1}^n\mathbb{I}\{R_i>r_i+\varepsilon_3\}\right]}{(n+1)\sqrt{\nu_3}}\text{ (by Markov's inequality)}\\
    &{} = \frac{n\p\left(R_i>r_i+\varepsilon_3\right)}{(n+1)\sqrt{\nu_3}}\text{ (by iid data)}\\
    &{} = \frac{n\p\left(|Y_i-\RFi(X_i)|>|Y_i-\rfi(X_i)|+\varepsilon_3\right)}{(n+1)\sqrt{\nu_3}}\\
    &{} \leq \frac{n\p\left(|Y_i-\RFi(X_i)-Y_i+\rfi(X_i)|>\varepsilon_3\right)}{(n+1)\sqrt{\nu_3}} \text{ (because $|a|-|b|\leq |a-b|$)}\\
    &{} = \frac{n\p\left(|\rfi(X_i)-\RFi(X_i)|>\varepsilon_3\right)}{(n+1)\sqrt{\nu_3}}\\
    &{} = \frac{n\p\left(|\rfi(X)-\RFi(X)|>\varepsilon_3\right)}{(n+1)\sqrt{\nu_3}} \text{ (because $X_i$ and $X$ are iid, and $\rfi$ and $\RFi$ do not depend on $X_i$)}\\
    &{} \leq \frac{n\nu_3}{(n+1)\sqrt{\nu_3}} \text{ (by concentration of measure)}\\
    &{} \leq \sqrt{\nu_3}.
\end{align*}
We then have
\[
\p(\widetilde r_{n+1} \leq q_{n,\alpha_3}\{R_i\} - \varepsilon_2-\varepsilon_3)\leq 1-\alpha_3+\frac{1}{n+1}+2\sqrt{\nu_2}+2\sqrt{\nu_3}.
\]

\subsection{Using the concentration of measure to control $|\rfall(X)-\RF(X)|$}
We write $R_{n+1}=|Y_{n+1}-\RF(X_{n+1})|$. Then 
\begin{align*}
    \p(R_{n+1} 
    &{} \leq q_{n,\alpha_3}\{R_i\}-\varepsilon_1-\varepsilon_2-\varepsilon_3)\\
    &{} = \p\left(R_{n+1}-\widetilde r_{n+1} + \widetilde r_{n+1}\leq q_{n,\alpha_3}\{R_i\}-\varepsilon_1-\varepsilon_2-\varepsilon_3\right)\\
    &{} \leq \p\left(R_{n+1}-\widetilde r_{n+1} < -\varepsilon_1\right) + \p\left(\widetilde r_{n+1}\leq q_{n,\alpha_3}\{R_i\}-\varepsilon_2-\varepsilon_3\right)\\
    &{} \leq \p\left(|R_{n+1}-\widetilde r_{n+1}| > \varepsilon_1\right) + 1-\alpha_3+\frac{1}{n+1} + 2\sqrt{\nu_2} +2\sqrt{\nu_3}\\
    &{} = \p\left(\left||Y_{n+1}-\RF(X_{n+1})|-|Y_{n+1}-\rfall(X_{n+1})|\right| > \varepsilon_1\right) + 1-\alpha_3+\frac{1}{n+1} + 2\sqrt{\nu_2} +2\sqrt{\nu_3}\\
    &{} \leq \p\left(|\rfall(X)-\RF(X)| > \varepsilon_1\right) + 1-\alpha_3+\frac{1}{n+1} + 2\sqrt{\nu_2} +2\sqrt{\nu_3}\text{ (because $||a|-|b||\leq |a-b|$)}\\
    &{} \leq 1-\alpha_3+\frac{1}{n+1} + \nu_1 + 2\sqrt{\nu_2} +2\sqrt{\nu_3}, \text{   (by concentration of measure)}
\end{align*}
which completes the proof. Again, for $\nu_1\in(0,1)$, the upper bound is $1-\alpha_3+\frac{1}{n+1}+O(\sqrt{\nu_{n,B}})$.

The proof of Theorem \ref{thm:PI_lower_bound} bears some resemblance to that of Theorem \ref{thm:PI_upper_bound} because of the symmetry consideration. For example, ``$\rfall(X)$ is close to $\rfi(X)$'' also means ``$\rfi(X)$ is close to $\rfall(X)$.'' Hence the probabilistic deviation bounds established above apply in two directions, and the upper bound can be proved. Still, we use the stability property once and the concentration of measure twice.

\section{Proof of Theorem \ref{thm:PI_asymptotic}}
We start from $\p(|Y-\RF(X)|\leq q_{n,\alpha}\{R_i\}+\varepsilon_{n,B})$. Denote $R=|Y-\RF(X)|$, and let its CDF be $F_n$, where we explicitly use $n$ to denote the size of the training set. Then we can rewrite $\p(|Y-\RF(X)|\leq q_{n,\alpha}\{R_i\}+\varepsilon_{n,B})$ as
\[
\p(|Y-\RF(X)|\leq q_{n,\alpha}\{R_i\}+\varepsilon_{n,B}) = \E[F_n(q_{n,\alpha}\{R_i\}+\varepsilon_{n,B})],
\]
where the expectation is with respect to $q_{n,\alpha}\{R_i\}$, which is random. By the assumption, when $n\geq n_0$, the family $\{F_n(t)\}_{n\geq n_0}$ is uniformly equicontinuous, which means that for any $\nu>0$, there exists some $\delta>0$ that is independent of $t$ and $n$, such that $|F_n(t')-F_n(t)|\leq \nu$ for all $|t'-t|\leq \delta$ and $n\geq n_0$. Then we have for $n\geq n_0$ and $\nu>0$ that
\begin{align*}
    & F_n(q_{n,\alpha}\{R_i\} +\varepsilon_{n,B}) - F_n(q_{n,\alpha}\{R_i\}) \\
    = {} & F_n(q_{n,\alpha}\{R_i\} +\varepsilon_{n,B})
    - F_n(q_{n,\alpha}\{R_i\} + \delta)
    + F_n(q_{n,\alpha}\{R_i\} + \delta)
    - F_n(q_{n,\alpha}\{R_i\})\\
    \leq {} & \sup_t \left[ F_n(t +\varepsilon_{n,B})
    - F_n(t + \delta) \right]
    + \sup_t \left [ F_n(t + \delta)
    - F_n(t) \right]\\
    \leq {} & \sup_t \left[ F_n(t +\varepsilon_{n,B})
    - F_n(t + \delta) \right] + \nu.
\end{align*}
However, $F_n(t)$ as a CDF is monotonically increasing, and for any $t$, $F_n(t +\varepsilon_{n,B}) - F_n(t + \delta) \leq 0$ as long as $\varepsilon_{n,B} \leq \delta$. By Corollary \ref{cor:epsilon_nu_to_0}, $\lim_{n\to\infty}\varepsilon_{n,B}=0$. So there exists some $n_1$ such that when $n\geq n_1$, we have $\varepsilon_{n,B} \leq \delta$.  Hence we conclude that for any $\nu>0$, for all $n\geq \max\{n_0,n_1\}$
\[
0\leq F_n(q_{n,\alpha}\{R_i\} +\varepsilon_{n,B}) - F_n(q_{n,\alpha}\{R_i\}) \leq \nu \text{ a.s.,}
\]
implying that
\[
\lim_{n\to\infty} \{F_n(q_{n,\alpha}\{R_i\} +\varepsilon_{n,B}) - F_n(q_{n,\alpha}\{R_i\})\} = 0 \text{ a.s.}
\]
As $F_n$ is bounded, then by the bounded dominance theorem, we have
\[
\lim_{n\to\infty} \E[F_n(q_{n,\alpha}\{R_i\} +\varepsilon_{n,B}) - F_n(q_{n,\alpha}\{R_i\})] =\E[\lim_{n\to\infty} \{F_n(q_{n,\alpha}\{R_i\} +\varepsilon_{n,B}) - F_n(q_{n,\alpha}\{R_i\})\}] =  0.
\]
Under the conditions of Corollary \ref{cor:epsilon_nu_to_0}, both $\varepsilon_{n,B}$ and $\nu_{n,B}$ go to 0 as $n\to\infty$, and the lower bound (\ref{eq:PI_lower_bound}) implies that
\begin{align*}
    1-\alpha &{} \leq \liminf_{n\to\infty} \E[F_n(q_{n,\alpha}\{R_i\}+\varepsilon_{n,B})]\\
    &{} =
    \liminf_{n\to\infty} \E[F_n(q_{n,\alpha}\{R_i\}+\varepsilon_{n,B}) - F_n(q_{n,\alpha}\{R_i\}) + F_n(q_{n,\alpha}\{R_i\})] \\
    &{} = \lim_{n\to\infty} \E[F_n(q_{n,\alpha}\{R_i\} +\varepsilon_{n,B}) - F_n(q_{n,\alpha}\{R_i\})] + \liminf_{n\to\infty} \E[F_n(q_{n,\alpha}\{R_i\})]\\
    &{} = \liminf_{n\to\infty} \E[F_n(q_{n,\alpha}\{R_i\})].
\end{align*}
Similarly, the upper bound (\ref{eq:PI_upper_bound}) implies that
\begin{align*}
    1-\alpha &{} \geq \limsup_{n\to\infty} \E[F_n(q_{n,\alpha}\{R_i\}+\varepsilon_{n,B})]\\
    &{} =
    \limsup_{n\to\infty} \E[F_n(q_{n,\alpha}\{R_i\}+\varepsilon_{n,B}) - F_n(q_{n,\alpha}\{R_i\}) + F_n(q_{n,\alpha}\{R_i\})] \\
    &{} = \lim_{n\to\infty} \E[F_n(q_{n,\alpha}\{R_i\} +\varepsilon_{n,B}) - F_n(q_{n,\alpha}\{R_i\})] + \limsup_{n\to\infty} \E[F_n(q_{n,\alpha}\{R_i\})]\\
    &{} = \limsup_{n\to\infty} \E[F_n(q_{n,\alpha}\{R_i\})]. 
\end{align*}
Combining these results, we have
\[
1-\alpha\leq \liminf_{n\to\infty} \E[F_n(q_{n,\alpha}\{R_i\})] \leq \limsup_{n\to\infty} \E[F_n(q_{n,\alpha}\{R_i\})]\leq 1-\alpha.
\]
Thus,
\[
\lim_{n\to\infty}\p(R\leq q_{n,\alpha}\{R_i\}) = \lim_{n\to\infty} \E[F_n(q_{n,\alpha}\{R_i\})] =  1-\alpha,
\]
completing the proof.

\end{document}